\documentclass[10pt,reqno]{article}

\usepackage{newtxtext,newtxmath}

\usepackage{mathtools}
\usepackage{amsthm}
\usepackage{amsfonts}
\usepackage{bm}
\usepackage{bbm}
\usepackage{mathrsfs}
\usepackage{nicefrac}
\usepackage{relsize}

\usepackage{thmtools}
\usepackage{thm-restate}

\usepackage[margin=1in]{geometry}
\usepackage{parskip}
\usepackage{microtype}

\usepackage{tikz}
\usepackage{wrapfig}
\usepackage{sidecap}
\usepackage{caption}
\usepackage{subcaption}
\usepackage{booktabs}
\usepackage{longtable}
\usepackage{multirow}

\usepackage{enumitem}

\usepackage{natbib}

\usepackage{authblk}
\usepackage{comment}
\usepackage{titletoc}
\usepackage{url}

\usepackage{xcolor}
\definecolor{linkblue}{RGB}{0,70,140}

\usepackage[colorlinks=true]{hyperref}
\usepackage[capitalize,nameinlink,noabbrev]{cleveref}

\hypersetup{
    citecolor=linkblue,
    linkcolor=linkblue,
    urlcolor=linkblue
}

\theoremstyle{definition}
\newtheorem{definition}{Definition}
\theoremstyle{plain}
\newtheorem{proposition}{Proposition}
\newtheorem{lemma}{Lemma}
\newtheorem{corollary}{Corollary}

\theoremstyle{remark}
 
\newcommand{\R}{\mathbb{R}}

\newcommand{\Z}{\mathbb{Z}}

\newcommand{\X}{\mathcal{X}}

\newcommand{\Pbb}{\mathbb{P}}
\newcommand{\norm}[1]{\left\lVert #1 \right\rVert}
\newcommand{\KL}{\mathrm{KL}}

\newcommand{\F}{\mathcal{F}}
\newcommand{\W}{\mathcal{W}}

\newcommand{\U}{\mathcal{U}}

\title{Specific Algorithmic Interpretability of Neural Networks: A Case Study on Textures}

\author[1]{Yanglin Zhang\textsuperscript{*}}
\author[2]{Anneke von Seeger\textsuperscript{*}}
\author[2]{Gilad Lerman}
\author[1]{Ron Levie}

\affil[1]{Technion -- Israel Institute of Technology, Faculty of Mathematics}
\affil[2]{University of Minnesota, School of Mathematics}

\date{}
\begin{document}

\maketitle

\begingroup
\renewcommand{\thefootnote}{\fnsymbol{footnote}}
\footnotetext[1]{Equal contribution.}
\endgroup

\begin{abstract}
We develop a principled framework for constructing neural networks whose specific parameter realizations admit an explicit algorithmic interpretation. Existing algorithm-inspired architectures can explain the computational structure of a network, yet after standard training the learned parameters need not retain a clear relation to the motivating algorithm. We address this gap as follows. First, we model each data point as a sample of a class-dependent stochastic process and assume that statistics of this process can be estimated from a single sample and these statistics are sufficient to distinguish the classes. We then construct a neural network whose initial parameters exactly implement an algorithm for estimating these statistics, making the network fully interpretable. 
 To account for mismatch between the idealized model and real data, we fine-tune this network while controlling its deviation from the algorithmic initialization. The trained network hence roughly retains the interpretation of the initial network. A PAC-Bayesian analysis yields a uniform generalization bound whose complexity term scales with the fine-tuning radius, providing a statistical motivation for our approach. 
We instantiate the framework for texture classification using the scattering transform  to estimate the discriminative statistics. 
\end{abstract}

\section{Introduction}

Classical algorithms are typically constructed as explicit sequences of computational operations. Each step is specified by design and serves a known role in transforming the input toward the desired output, making the underlying computational procedure transparent. Neural networks, in contrast, are produced through data-driven learning processes. Their behavior is determined jointly by the computational structure of the network and a collection of parameters optimized from data. Although the forward computation of a trained network is fully specified mathematically, understanding how these learned parameters and computations collectively give rise to the underlying decision process is often difficult.

The boundary between classical algorithms and neural networks, however, is not absolute. 
Many algorithms can be expressed as a sequence of operations that naturally correspond to the layers or modules of a neural network. This connection is exploited, for example, in algorithm unfolding, where the steps of an iterative algorithm are mapped to network layers, and is closely related to algorithmic alignment, where the computational structure of a network is chosen to reflect the decomposition of a target algorithm \citep{Monga2021Unrolling,xu2020what}.

An algorithmic interpretation of the architecture, however, does not necessarily provide an interpretation of a particular trained network. Consider a family of networks parameterized by weights. It may be possible to assign the same algorithmic interpretation to every member of this family, even though different parameter values instantiate different computations. Such an interpretation explains the algorithmic role of the architecture, but does not necessarily explain what is encoded by the learned parameters. We refer to this notion as \emph{generic algorithmic interpretability}.  
 Examples of such architecture-level interpretations include viewing neural
networks as unfolded sparse-coding or pursuit algorithms
\citep{Papyan2017CNN,Sulam2020CNN}, transformer blocks as iterations of
sparse rate reduction \citep{Yu2023WhiteBoxTransformers}, deep networks through the
information bottleneck principle \citep{Tishby2015InfoBottleneck}, and deep architectures
through the variational renormalization group
\citep{Mehta2014Renormalization}.

In this work, we seek a stronger notion, which we call \emph{specific algorithmic interpretability}. Here, the interpretation depends on the particular values of the parameter: different values of the parameters can be related to different explicit algorithmic objects or operations. Thus, rather than asking only what algorithmic template a network architecture follows, we ask what particular algorithm is implemented by a trained network in view of the values of its parameters. Our goal is to develop a principled framework for constructing neural networks whose learned parameters admit such an explicit algorithmic interpretation, allowing the underlying decision process of the trained network to be understood.

Concretely, we begin with a predefined network whose parameters implement an explicit algorithm. To decide what the target of the algorithm should be, we first propose a model of the data.
We model observations as realizations of class-dependent stochastic processes and interpret the algorithm as estimating,
from a single observation, statistics of its generating process that are sufficient to distinguish the class. 
This gives the predefined parameter realization a specific algorithmic interpretation: its computation can be
understood in terms of the statistics it estimates. 
More broadly, this leads to a stochastic view of pattern recognition in which recognition amounts to
inferring discriminative properties of an observation's underlying generating process.

A difficulty arises when such an algorithmically motivated network is trained in the usual way: unconstrained optimization may move its parameters to a realization with no clear relation to the original algorithm. 
We call this the \emph{algorithmic alignment gap}. We address it by initializing the network as the interpretable algorithm and fine-tuning on real data by controlling the distance traveled in parameter space. A PAC-Bayesian analysis shows that,
under a local stability condition, the resulting generalization bound scales with this fine-tuning radius, providing a statistical justification for training large neural networks in the low-data regime.

We then instantiate the framework in texture classification, where ergodicity allows statistics of an underlying stationary process to be estimated from a single image. 
The scattering transform provides an explicit interpretable algorithm that we realize as predefined neural networks and use as initializations for subsequent fine-tuning.

\subsection{Related Work}

\paragraph{Algorithmic interpretations of neural networks.}
A broad literature seeks to make learned models interpretable, including
intrinsic approaches, which design models to be interpretable by construction
, and post-hoc approaches, which explain already trained models
\citep{Rudin2018StopEB,Arrieta2019ExplainableAI}. Our work belongs to the former setting and
focuses specifically on algorithmic interpretations of neural networks.
Closely related are algorithm-unfolding methods, which map iterations of an optimization algorithm to
network layers \citep{Monga2021Unrolling}; a canonical example is LISTA, which
unfolds the Iterative Shrinkage-Thresholding Algorithm and learns its operators
from data \citep{Gregor2010LearningFA}. Related algorithmic interpretations
have also been developed for CNNs and transformer-like architectures
\citep{Papyan2017CNN,Yaodong2024WhiteBoxTransformers}.
These works typically provide algorithmic interpretations of the computational structure shared across a family of parameterizations, which we refer to as \emph{generic algorithmic interpretability}. 
For example, a sparse-coding network may still be interpreted as performing sparse inference after its dictionary has been learned, even though this interpretation does not by itself explain the
particular learned dictionary. 
Our work instead targets \emph{specific algorithmic interpretability}, where the particular learned parameters themselves can be related to explicit algorithmic objects and operations, thereby revealing the particular algorithm implemented by the trained network. Further connections to interpretability literature are discussed in Appendix~\ref{app:related_work}.

\paragraph{Neural Collapse}
\label{realted_work:neural_collapse}
Neural collapse (NC) is an empirical phenomenon that emerges in the penultimate-layer representations of deep classification networks during the terminal phase of training \citep{papyan_prevalence_2020}. In the input domain, samples from different classes may be highly entangled and difficult to separate directly. During training, the network progressively disentangles these inputs by learning representations in which features from the same class become concentrated while those from different classes become increasingly separated. Under neural collapse, this disentangling process culminates in representations for which the classes are readily separable by the final linear classifier. Subsequent theoretical works have shown that, when the feature extractor is treated as infinitely expressive, neural collapse arises as an optimal solution of the classification objective \citep{e_emergence_2022, lu_neural_2022, tirer_extended_2022, nguyen_memorization-dilation_2022, mixon_neural_2022}. These results suggest that neural collapse can be viewed as a limiting behavior toward which training drives the learned representation, progressively transforming an initially entangled classification problem into one that is simple to solve with a linear classifier. In our approach, the target of the initial algorithm is to promote neural collapse.

\paragraph{PAC-Bayesian bounds for neural networks}
PAC-Bayesian theory provides generalization bounds in terms of a distribution over predictors, typically controlled by the KL divergence between a learned posterior and a chosen prior \citep{mcallester1999pac,mcallester2003simplified,langford2002pac}. For neural networks, \citet{neyshabur_pac-bayesian_2018} derived PAC-Bayesian margin bounds by relating parameter perturbations to changes in the network output through layer-wise spectral norms. Subsequent work extended and refined such analyses for convolutional architectures and other structured neural networks \citep{behboodi_pac-bayesian_2022}. These results provide an alternative to classical uniform complexity bounds, which can become vacuous for overparameterized networks \citep{shalev2014understanding}.

\paragraph{The wavelet scattering transform}
\label{sec:wavelet_related}

Wavelet-based models provide explicit multiscale statistics for stationary
images.\\
\textit{Wavelet transform.}
The wavelet transform represents images as superpositions of local details at different scales.
Let $\{\psi_\lambda\}_{\lambda\in\Lambda}$ be a so called family of complex wavelets. Here, $\lambda$ indexes scale and orientation, and each wavelet $\psi_\lambda$ is typically chosen as a local bump with oscillations at the direction and rate parameterized by $\lambda$. The wavelet transform represents
an image $x$ through the coefficients
$
x\star\overline{\psi_\lambda}
$, where $\star$ denotes correlation, 
which describe its local content at different scales and orientations. We also
let $\phi_J$ denote a low-pass filter performing spatial averaging at scale
$2^J$. More details on the wavelet transform are given in Appendix~\ref{app:wavelet_transform}.

\textit{Wavelet scattering transform.}
The wavelet scattering transform constructs multiscale statistics by cascading wavelet convolutions with complex modulus nonlinearities, followed
by spatial averaging \citep{mallat2012groupinvariantscattering}. For a scattering path
$p=(\lambda_1,\ldots,\lambda_q)$, define
$
U[\emptyset]x=x
$
and recursively
$
U[p]x
=
|U[(\lambda_1,\ldots,\lambda_{q-1})]x*\psi_{\lambda_q}|.
$
The windowed scattering coefficient associated with $p$ is
$
S_J[p]x
=
U[p]x*\phi_J.
$ More details on the scattering transform are given in \cref{app:scattering_transform}. 
\citet{mallat2012groupinvariantscattering} also formulates scattering for stationary stochastic processes through
the expected scattering coefficients
$
\mathbb E[U[p]X].
$
The windowed scattering coefficients $S_J[p]x$ can therefore be viewed as
estimates of these statistics  from one single realization. In particular, when $\phi_J$ is normalized to integrate to one,
they are unbiased estimators of $\overline S_X(p)$; see the explanation of \cref{eq:expected_scattering} in \cref{app:scattering_transform} for more details. 
It provides a predefined family of statistics for stationary processes, which we use as the selected statistics in our stochastic model to construct an explicit feature extractor for texture classification.

\section{Specific algorithmic interpretability  via a model of perception}

In this section we present our  general approach for specific algorithmic interpretability.

\paragraph{Notation.}
For an integer $m\geq1$, let $[m]:=\{1,\ldots,m\}$. We consider a
$C$-class classification problem with input space
$
\mathcal X:=\{x\in\mathbb R^n:\|x\|_2\leq B\}
$
and representation space $(\mathcal Z,d)$. Let $F_\Theta$ denote a neural network parameterized by
$\Theta$. Depending on the context, $F_\Theta$ denotes either the feature
extractor
$
F_\Theta:\mathcal X\rightarrow\mathcal Z,
$ or the corresponding predictor with a single-layer linear head
$
F_\Theta:\mathcal X\rightarrow\mathbb R^C.
$
For $j\in[C]$, we write $F_\Theta(x)[j]$ for the $j$-th coordinate of
$F_\Theta(x)$.
When the
parameters consist of $d$ matrix-valued blocks,
$
\Theta
=
\operatorname{vec}\bigl(\{W_i\}_{i=1}^d\bigr),
$
so that
$
\|\Theta\|_2^2
=
\sum_{i=1}^d\|W_i\|_F^2.
$
Throughout, $\|\cdot\|_F$ denotes the Frobenius norm, while
$\|\cdot\|_2$ denotes the Euclidean norm for vectors and the spectral norm for
matrices. We use $\|\cdot\|_\infty$ for the maximum norm.

\subsection{From neural collapse to data disentangling}
\label{subsec:data_disentangling}

\paragraph{Neural collapse as a geometric objective.}
We use neural collapse as a guiding geometric objective for understanding the computation performed by the neural network feature extractor; see
Section~\ref{realted_work:neural_collapse} for background. Rather than explaining all properties associated with neural collapse, we focus on its basic geometric principle: representations should concentrate observations from the same class while separating observations from different classes. If $F_\Theta$ denotes the neural network feature extractor, our question is how the computation
 $x \longmapsto F_\Theta(x)$ 
transforms an initially entangled input signal space into a class-organized representation that is simple to separate. We therefore seek explicit computations for which within-class concentration and between-class separation can be established mathematically.

\paragraph{Data disentanglers.}
We begin with the simplest setting in which the computation of a feature
extractor can be understood explicitly: a neural network whose parameters are
chosen to implement a prescribed algorithm rather than learned from data.
Let $F_{\Theta}:\mathcal X\rightarrow\mathcal Z$ denote a generic neural network feature
extractor. We are interested in a particular
predefined parameter realization $\Theta_0$ for which the computation
$
x\longmapsto F_{\Theta_0}(x)
$
has an explicit algorithmic interpretation.

We call $F_{\Theta_0}$ a \emph{data disentangler} if it maps observations to a
class-separating representation under an appropriate model of the data.
Unlike clustering, which organizes observations using relations among the
samples in a dataset, a data disentangler applies the same transformation to
each observation independently. Any class information encoded in the representation must therefore be recoverable from properties of the individual observation itself.

This motivates a stochastic view of the data. We regard each observation as a
realization of an underlying stochastic process, and interpret the role of the
data disentangler as estimating from that realization statistics of the
generating process that are sufficient to distinguish its class.

\subsection{A transparent stochastic data model}\label{subsec:stochastic_model}

Let $\mathfrak S$ denote a family of stochastic processes, where each
$X\in\mathfrak S$ induces a probability law $P_X$ on the observation space
$\mathcal X$. We place a probability measure $\mu$ over $\mathfrak S$, 
noting that there is a natural measurable structure over $\mathfrak S$. 
The data distribution $P$ is then modeled as the mixture
\begin{equation}
\label{eq:mixture_processes}
P(A)
=
\int_{\mathfrak S} P_X(A)\,d\mu(X),
\qquad
A\subseteq\mathcal X.
\end{equation}
Equivalently, an observation is generated in two stages: first a stochastic
process $X\sim\mu$ is selected, and then a realization $x\sim P_X$ is drawn
from that process.

For a classification problem with $C$ classes, we assume that the process
space is partitioned into disjoint subsets
$
\mathfrak S
=
\mathfrak S_1\cup\cdots\cup\mathfrak S_C,
$
where processes in $\mathfrak S_c$ generate observations belonging to class
$c$. Let
$
\mu_c=\mu(\,\cdot\,\mid\mathfrak S_c)
$
denote the distribution of processes conditioned on class $c$. We denote by
\(Y_c\) the mixture process associated with class \(c\), whose induced law is
\begin{equation}
\label{eq:class_mixture_processes}
P_c(A)
=\int_{\mathfrak S_c} P_X(A)\,d\mu_c(X).
\end{equation}
Thus, variation within a class may arise both from randomness within an
individual process and from variation among the processes belonging to that
class.

\paragraph{Assumption: estimable class-separating statistics.}
We assume that there exists a metric space $(\mathcal Z,d)$ and a map
$
s:\mathfrak S\rightarrow\mathcal Z
$
that associates each stochastic process $X$ with a collection of statistics
$s(X)$. For each class, define
$
\mathcal A_c
=
\{s(X):X\in\mathfrak S_c\}.
$
We assume that these sets are separated across classes, so that the statistics
$s(X)$ are sufficient to determine the class of the generating process.

We further assume that the processes in $\mathfrak S$ are \emph{ergodic}, in the sense that
these statistics can be estimated from a single realization
$x\sim P_X$. In particular, there exists an estimator $\widehat s$ such that
\begin{equation}
\label{eq:statistic_consistency}
\widehat s(x)
\approx
s(X)
\end{equation}
whenever $x\sim P_X$,
in an appropriate asymptotic regime.

Under these assumptions, the predefined network becomes a data disentangler
when its parameters $\Theta_0$ are chosen so that
\begin{equation}
\label{eq:predefined_disentangler}
F_{\Theta_0}(x)
=
\widehat s(x).
\end{equation}
Indeed, if $x$ is generated by a process $X\in\mathfrak S_c$, then
$F_{\Theta_0}(x)$ approximates $s(X)\in\mathcal A_c$. The network therefore maps
individual observations to a class-separating representation determined by
statistics of their underlying generating processes. Moreover, the parameters
$\Theta_0$ have an explicit algorithmic meaning: they implement a prescribed
procedure for estimating these statistics.

\paragraph{The stochastic pattern recognition hypothesis.}
The transparent stochastic model suggests a broader view of pattern recognition.
Suppose an observation $x$ is generated by a stochastic process
$X\in\mathfrak S$. A conventional view of pattern recognition treats $x$ itself as the
object to be classified. Here we adopt a different perspective: $x$ is viewed
as a realization that carries information about its underlying generating
process. In this
view, pattern recognition is therefore formulated as inferring from $x$ the statistics
$s(X)$ that distinguish this process from the relevant alternatives.
This motivates what we call the \emph{stochastic pattern recognition hypothesis}: 
\begin{quote}
   \emph{pattern recognition consists of inferring, from an individual observation, statistics that distinguish its underlying stochastic generating process from the other processes relevant to the task.} 
\end{quote}
Under this view, the fundamental objects being distinguished are generating mechanisms rather than the particular samples appearing in a finite dataset. We conjecture that neural feature extractors can be interpreted as progressively estimating increasingly discriminative statistics of these mechanisms, thereby transforming the input into a representation in which recognition becomes simple.

This statistical view has strong precedents in texture analysis. Julesz's classical work proposed that human texture perception can be characterized through statistics of an underlying stochastic image process \citep{Julesz1962}. In a later formulation of Julesz's hypothesis, one assumes that there exists a collection of functions $\{\phi_k\}_{k=1}^K$ such that two stochastic texture processes $P$ and $Q$ are perceptually indistinguishable if and only if
$$
\mathbb E_{X\sim P}[\phi_k(X)]
=
\mathbb E_{X\sim Q}[\phi_k(X)],
\qquad k=1,\ldots,K.
$$ 
 This
statistical viewpoint has subsequently appeared in increasingly rich texture models and scattering-based representations of stationary processes \citep{PortillaSimoncelli2000,mallat2012groupinvariantscattering}, and related formulations arise in the classification of stationary ergodic processes \citep{RyabkoRyabko2010}. In Section~\ref{sec:texture_classification} we instantiate our setting for texture classification.

\subsection{The algorithmic alignment gap}
\label{subsec:alignment_gap}

The transparent stochastic model provides a particular parameter realization
$\Theta_0$ for which the network $F_{\Theta_0}$ implements an explicit
data-disentangling algorithm. This gives the predefined network a specific
algorithmic interpretation. However, this property is generally lost under standard end-to-end training. Algorithm unfolding and related approaches construct a network architecture from the computational structure of an algorithm, but the resulting network is then typically fully trained without requiring its learned parameters to remain related to the parameters that implement the original algorithm. Consequently, the architecture may retain a generic algorithmic interpretation while the particular trained parameter realization doesn't admit an explicit algorithmic interpretation.

We refer to this discrepancy as the \emph{algorithmic alignment gap}: an architecture may be aligned with an interpretable algorithm, while the parameters obtained after training need not retain a specific algorithmic interpretation. Bridging this gap therefore requires controlling not only the architecture of the network, but also how its parameters change during learning.

\subsection{Bridging the algorithmic alignment gap: fine tuning}
\label{subsec:fine_tuning}

To address the algorithmic alignment gap, we propose using the interpretable
realization $\Theta_0$ as the initialization for learning and
\emph{fine-tuning} the network, i.e., restricting how far its parameters move
from this initialization. Let $\Theta$ denote the parameters after
fine-tuning. The quantity
$
\|\Theta-\Theta_0\|_2
$
 measures how far learning has moved the network from the
algorithmic realization. 
One motivation of fine-tuning is that $\Theta_0$ is constructed under an
idealized stochastic model, whereas real data need not satisfy this model
exactly. Some adaptation is therefore desirable to compensate for
this mismatch. Accordingly, the predefined network $F_{\Theta_0}$ is not required to perform competitively with modern trained models; rather, we suppose that a well-performing predictor can be reached without moving too far from $\Theta_0$.

Our theoretical result gives a statistical reason to restrict this movement.
For $R>0$, define the neighborhood
$
B_R(\Theta_0)
=
\{\Theta:\|\Theta-\Theta_0\|_2\leq R\}.
$
We show that for networks in this neighborhood,
the generalization bound, which bounds the gap between the statistical and empirical risks,  
scales linearly with $R$
and inversely with $\sqrt{M}$, where $M$ is the number of training samples.
Thus, restricting the distance from the algorithmic initialization controls an
aspect of the learned model beyond its algorithmic interpretation: its
generalization bound. 
This is the sense in which fine-tuning bridges the algorithmic alignment gap.
The transparent stochastic model gives us an explicit, data-independent
parameter realization $\Theta_0$ with a known algorithmic meaning. By
restricting learning to $B_R(\Theta_0)$, we can retain most of the interpretation of
this reference point while analyzing the generalization gap.

To formulate our result, define the classification margin
$
m_\Theta(x,y)
=
F_\Theta(x)[y]
-
\max_{j\neq y}F_\Theta(x)[j].
$
Let $R_0(F_\Theta)$ denote the statistical classification risk and
$\widehat R_{\gamma}(F_{\Theta})$ the empirical $\gamma$-margin risk.
Our analysis uses what we call an $(R,\tau,\epsilon)$-\emph{stability} condition. Informally, it
requires that all networks within distance $R$ of $\Theta_0$ change their
outputs by at most $\epsilon$ under parameter perturbations of size at most
$\tau$. Formal definitions and the PAC-Bayesian derivation are given in
Appendix~\ref{app:general_ft_bound}.

\begin{restatable}[Uniform Generalization Bound Around a Data-Independent Initialization]{corollary}{}\label{coro:general_ft_bound}
    Let $\F$ be a parametrized hypothesis family, and let $\Theta_0=vec(\{W_i\}_{i=1}^d)$ be an initialization independent of the training data, where $W_i\in\R^{m_i\times n_i}$, and define $h:=\max_{i\in[d]}\{m_i, n_i\}.$ Fix $R,\gamma,\tau>0$, and suppose that $\F$ is ($R,\tau,\frac{\gamma}{4}$)-stable under bounded parameter perturbations around $\Theta_0$. 
    Then, for any $\delta \in (0,1)$, with probability at least $1-\delta$ over the draw of a training set of size $M$,
    the following holds for all $\Theta\in B_R(\Theta_0)$:
    \begin{equation}
    R_{0} (F_{\Theta}) \leq \widehat R_{\gamma}(F_{\Theta}) + 4 \sqrt{ \big( R^2 h \ln(4dh)/\tau^2 + \ln (6 M/\delta) \big) / \left( M - 1 \right)} \label{eq:general_ft_bound}    
    \end{equation}
\end{restatable}

\paragraph{Implications for training.}
We use the bound primarily as an \emph{a priori design principle} rather than as a post-training numerical certificate. We note that after training, predictive performance can be assessed directly on held-out data. Hence,  evaluating the numerical value of a worst-case PAC-Bayesian bound is not the purpose of our analysis. Corollary~\ref{coro:general_ft_bound} shows that for a fixed architecture and perturbation-stability scale $\tau$, the generalization gap scales with
$
O(\frac{R}{\sqrt{M}}).
$
Hence, a smaller adaptation radius  yields a tighter bound; equivalently, for a fixed target
generalization gap, a smaller adaptation radius allows the same guarantee to be obtained with fewer
training samples. This suggests fine-tuning an algorithmically meaningful initialization can
be particularly useful in the low-data regime. It therefore motivates to control $\norm{\Theta-\Theta_0}_2.$ Such control can be imposed explicitly through a regularizer
$
\lambda\norm{\Theta-\Theta_0}_2^2,
$
or encouraged by a conservative optimization procedure, such as a sufficiently
small learning rate. 

Importantly, the bound does not require the predefined disentangler itself to achieve competitive performance on the real data. Its role is to provide a data-independent and algorithmically interpretable reference point from which a good predictor can be reached with limited adaptation. This leaves substantial freedom in the construction of the initial algorithm.

\section{Specific Algorithmic Interpretability  In Texture classification}
\label{sec:texture_classification}

In this section  we instantiate the above theory on the concrete example of texture classification. 

\subsection{Textures as stationary processes with subset statistics}
\label{subsec:texture_model}

A texture is an image exhibiting repeated local structure without requiring
strict periodicity. Textures are commonly modeled as realizations of stationary random fields
\citep{Julesz1962,PortillaSimoncelli2000,mallat2012groupinvariantscattering}.  
Let $Y=(Y(u))_{u\in\mathbb Z^2}$ be a random field. For $g\in\mathbb Z^2$, define the translation action on a random field by
$
(g.Y)(u):=Y(u-g)
$.  
A random field $Y$ is \emph{stationary} if its law is invariant under this action:
$
g.Y\overset{d}{=}Y
$
for every $g\in\mathbb Z^2$.  
A texture image is viewed as a finite  observation of this field in a spatial rectangular window. We next discuss the ergodic decomposition of the stationary process \(Y\). For consistent notation, we denote a generic ergodic process arising in this decomposition by \(X\).

Ergodicity has a precise definition in probability  theory. We give here an equivalent definition which is more in line with our discussion, i.e., the empirical average over any integrable function of the array converges to the expected value of this function. The equivalent formulation follows the so called \emph{Birkhoff's ergodic theorem} \citep[Sec.~1.2]{krengel2011ergodic}.
\begin{definition}
   The process $X=(X(u))_{u\in\mathbb{Z}^2}$ is called \emph{ergodic} if for every integrable function $f$ with respect to the law of $X$, i.e. $f\in L_1(P_X)$, we have
   \begin{equation}
       \label{eq:ergodi} \lim_{N\rightarrow\infty }\frac{1}{4N^2}\sum_{g\in [-N,N]\times [-N,N]} f(g.X) \stackrel{\text{a.s.}}{=} \mathbb{E}[f(X)],
   \end{equation}
where expected value is with respect to the law of $X$.  
\end{definition}

A stationary process need not be ergodic. However, the \emph{ergodic decomposition
theorem} \citep{Viana2016ErgodicDecomposition} states that any stationary process is a mixture of ergodic processes in the sense of (\ref{eq:mixture_processes}), where each $P_X$  is ergodic. Hence, stationary processes are well suited to our analysis.

One can think about the set of all integrable functions $f$ in (\ref{eq:ergodi}) as the set of sufficient statistics.\footnote{In fact, this claim can be made precise. By taking all indicator functions as $f$, one can extract the ergodic components.}  More accurately, the set of statistics is 
$\{X\mapsto \mathbb{E}[f(X)]\}$. However, this set is too large to be of any practical use. Instead, practical definitions of statistics consider subsets of the integrable functions as the statistics. By taking a subset $\mathcal{S}$ of $L_1(P_X)$, we reduce the resolution of the statistics. Namely, there are separate ergodic components with the same set of $\mathcal{S}$ statistics. One hence has to carefully choose a meaningful subset statistics for the task, in the sense that images with the same statistics tend to be in the same class.

\paragraph{Finite observations.}
The stochastic processes above are defined on the full lattice
$\mathbb Z^2$, whereas an observed texture image has finite spatial support.
We therefore model an image as a finite restriction as follows. The generative procedure first produces an
infinite random field $X$ as described above, after which only its restriction
$
X|_{\Lambda}
$
to a predefined finite rectangular domain
$\Lambda\subset\mathbb Z^2$ is observed. 
Consequently, the limiting statistics in
\eqref{eq:ergodi} cannot be computed exactly from an image. In practice, they
are estimated by averaging over the observed domain $\Lambda$. We assume that
$\Lambda$ is sufficiently large for these finite spatial averages to provide
good approximations to the corresponding ergodic limits.

The preceding discussion gives the texture-specific realization of the
stochastic model introduced in Section~\ref{subsec:stochastic_model}. 
In principle, a sufficiently rich collection of statistics can distinguish
the underlying ergodic components. In practice, however, only a finite subset
can be computed. The choice of this subset determines which differences
between generating processes are retained. For classification, we therefore
seek statistics that can be estimated reliably from a finite image while retaining the distinctions relevant for separating texture classes.

Classical statistical approaches to texture recognition characterize
processes through moments and correlations. Because higher-order moments
quickly become computationally expensive, practical methods have traditionally
focused largely on first- and second-order statistics. More expressive
representations can be obtained by choosing structured nonlinear statistics.
In what follows, we consider wavelet scattering transform as a concrete example \citep{mallat2012groupinvariantscattering}. 

\subsection{Wavelet scattering as a data disentangler}
\label{subsec:scattering_disentangler}

 We now instantiate the statistical construction above using the wavelet
scattering transform, reviewed in Section~\ref{sec:wavelet_related} and
described in detail in Appendix~\ref{app:scattering_transform}.
Let $\mathcal P$ denote the collection of scattering paths used in the
representation. For each path $p\in\mathcal P$, define
$
f_p(X):=U[p]X(0).
$
Since wavelet filtering and the pointwise modulus commute with translations,
evaluating $f_p$ on translated realizations corresponds to evaluating
$U[p]X$ at different spatial locations. Hence each $f_p$ is one of the
functions to which the ergodic averaging principle in
\eqref{eq:ergodi} applies. 
The windowed scattering coefficient
$
S_J[p]x(u)=U[p]x*\phi_J(u)
$
is precisely a weighted spatial average of these translated evaluations.
Unlike the limiting average in \eqref{eq:ergodi}, this average is weighted against the scaling function $\phi_J$. When $\phi_J$ is sufficiently high scale, integrating against it provides a finite-image approximation to the
corresponding process statistic
$
\mathbb E[f_p(X)].
$
Thus, scattering amounts to selecting the finite family 
$
\mathcal S_{\mathrm{scat}}
=
\{f_p:p\in\mathcal P\}
$, 
from the much larger collection of functions appearing in the ergodic
decomposition, and acts as a data disentangler. 
In Appendix~\ref{app:scattering_parameter_stability}, we  show that this disentangler satisfies the local parameter-stability condition required by Corollary~\ref{coro:general_ft_bound}.

\paragraph{Qualitative assessment through texture synthesis.}
Before evaluating the resulting representation for classification, we first
inspect whether the selected scattering statistics retain meaningful texture
information. Given a reference texture $x$, we synthesize an image
$\widetilde{x}$ whose scattering statistics match those of $x$ by solving
$
\widetilde{x}
\approx
\arg\min_z
\left\|
S_J(z)-S_J(x)
\right\|_2^2. 
$ Here, we use the micro-canonical approaches proposed by ~\citet{bruna2018Microcanonical};  
 implementation details are given in Appendix~\ref{app:scattering_synthesis}. 
We then visually inspect whether the synthesized image depicts a texture
similar to the reference. If matching the statistics consistently reproduces
the characteristic appearance of the reference texture, this suggests that the
selected statistics preserve perceptually relevant properties of the
underlying texture process.
\Cref{fig:texture_generation} shows representative results of the synthesis.

\begin{figure}[t]
    \centering
    \includegraphics[width=0.5\linewidth]{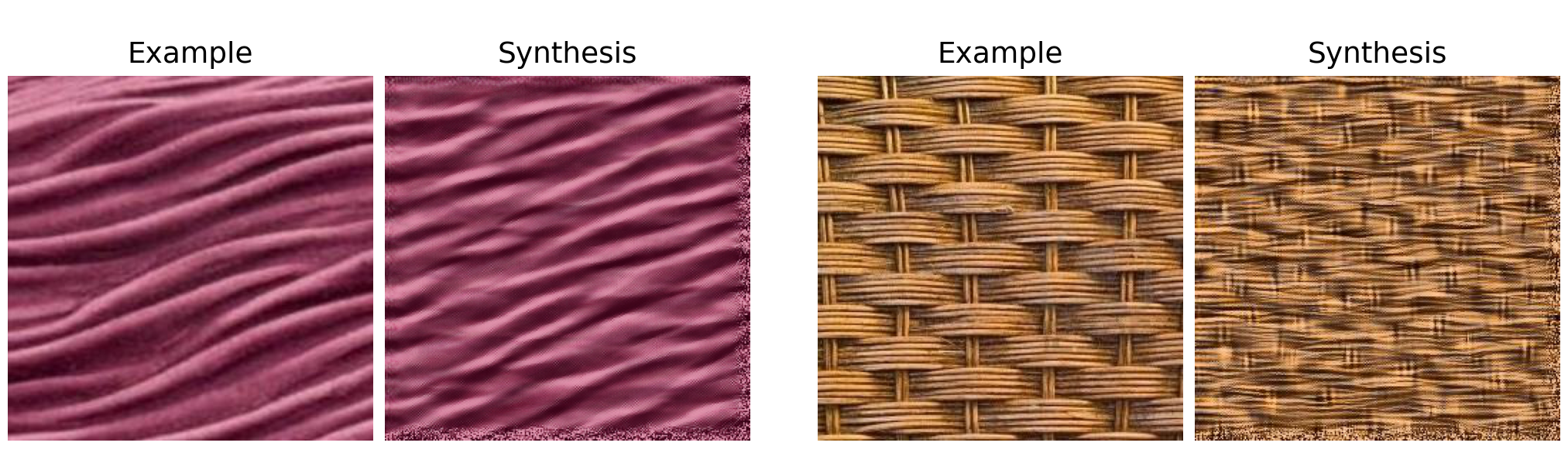}
    \caption{
    Texture synthesis by matching wavelet scattering statistics.
Synthesized images show that the selected scattering statistics retain perceptually meaningful texture information.
    }
    \label{fig:texture_generation}
\end{figure}

\section{Experiments}
\label{sec:texture_recognition_exp}
We realize the predefined scattering transform as a trainable convolutional network. At initial parameters $\Theta_0$, the feature extractor exactly implements the scattering transform, with wavelet
convolutions applied independently to each input channel. 
During fine-tuning, the filters can adapt and information across channels can be combined. 
\citet{gauthier_parametric_2022} propose a related approach in which only a small number of parameters controlling the filters are learned; see Appendix~\ref{app:parametric_scattering}.
We compare against three baselines: the same
architecture with random initialization, the original fixed scattering
transform, and the learnable scattering-network block architecture proposed by
\citet{zarka_separation_2021} (ScatNet B.A.). ScatNet B.A. keeps the wavelet
filters and scattering modulus fixed, while adding learnable $1\times1$
convolutions to combine and reduce channels, followed by learnable tight-frame
analysis layers. All methods use a single linear classification head. 
Together, these baselines isolate the benefit of fine-tuning from a predefined scattering realization, relative to random initialization, no adaptation, and an alternative trainable scattering architecture.
Model details are given in Appendix~\ref{app:architectures}. 

We evaluate on DTD, EuroSAT, MinC-2500, and SuroTex
\citep{cimpoi14describing,helber2019eurosat,bell2015material, putra2025surotex}.
Classical texture benchmarks such as KTH-TIPS and CUReT are not used since they are relatively small, and fixed scattering representations already perform strongly on them, leaving limited room to study the benefit of fine-tuning\citep{Fritz2004KTHTIPS, Dana1999Curet}.
SuroTex instead provides a large-scale texture dataset, while DTD, EuroSAT, and MinC-2500 introduce increasingly richer variability beyond classical texture images. Together, these datasets allow us to study the regime between textures well described by the predefined model and more general images for which adaptation becomes important. 
We also evaluate performance across different training-set sizes to study sample
efficiency; dataset splits and sampling details are given in
Appendix~\ref{app:datasets_split}.
For our architecture, we set the fine-tuning regularization parameter \( \lambda \) with values of \( 0, 0.1, 1.0 \).
For ScatNet B.A. and the randomly initialized architecture, we set \( \lambda = 0 \) to avoid anchoring to a prior that is not meaningful.

We report test accuracy and generalization gap in this section. We further conduct an ablation study to separate the effect of regularizing the fine-tuning distance from that of the scattering prior, and analyze the effect of fine-tuning on neural collapse and interpretability; these additional results are provided in \cref{app:more-results}. The ablation study shows that anchoring the model to the scattering prior improves accuracy compared with either no regularization or standard weight decay, while the interpretability analysis shows that fine-tuned filters remain wavelet-like.

\begin{figure}
    \centering
    \includegraphics[width=0.8\linewidth]{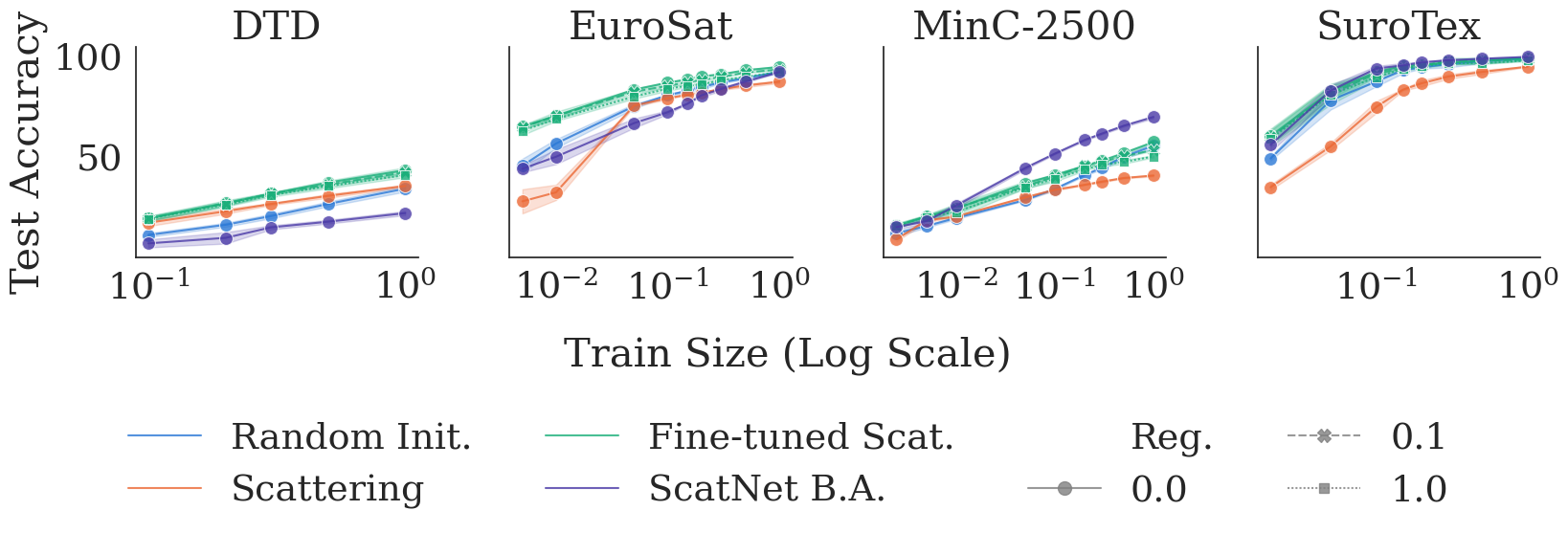}
    \caption{Percentage of data used during training versus test accuracy. Color denotes method and line style denotes regularization strength. The proposed method gives the best performance when the amount of training data is small.}
    \label{fig:acc-vs-train-size}
\end{figure}

\Cref{fig:acc-vs-train-size} displays the percentage of available data used during training versus test accuracy, calculated as the number of correct predictions out of total predictions on the test dataset.
\Cref{coro:general_ft_bound} focuses on controlling the generalization gap and does not directly address test accuracy.
However, having low generalization gap is typically related to high performance, so long as the training accuracy is high, as we demonstrate in \cref{fig:test-acc-vs-gen-gap}. 
\cref{fig:acc-vs-train-size} demonstrates that 
 the fine-tuned scattering network outperforms all baseline methods when training data is extremely limited.
On the least texture-like dataset, MinC-2500, ScatNet B.A. outperforms the fine-tuned scattering network as training data increases.
On the remaining datasets, the fine-tuned scattering network either performs the best or performs competitively with baseline methods as data increases.

\begin{figure}
    \centering
    \includegraphics[width=0.8\linewidth]{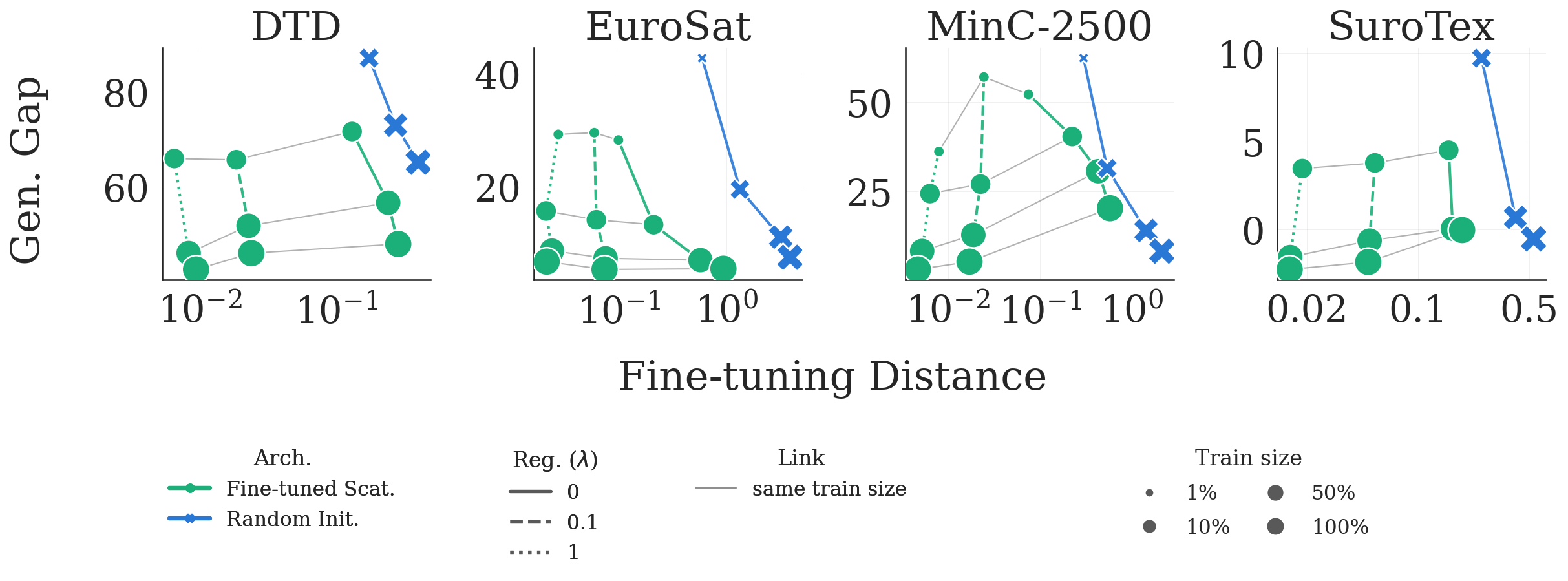}
    \caption{Fine-tuning distance  versus generalization gap. Color denotes method, line style denotes regularization strength, and size denotes amount of training data used. Gray lines between fine-tuned scattering networks trained on the same data size allow comparison across regularization strengths. Smaller fine-tuning distance correlates with lower generalization gap.}
    \label{fig:filter-dist-vs-gen-gap}
\end{figure}

\Cref{fig:filter-dist-vs-gen-gap} displays the fine-tuning distance against generalization gap, where fine-tuning distance is the \( L^2 \) distance between initial and final parameters, normalized by the initial parameter norm.
As predicted by \cref{coro:general_ft_bound}, the generalization gap tends to decreases as the fine-tuning distance decreases.
The effect is strongest on MinC-2500, the least texture-like dataset, suggesting that control of the fine-tuning distance helps most  when the data departs from the theoretical model and the parameters would otherwise travel further.
The randomly initialized network does not follow the trend  established by the fine-tuned scattering network, and has a higher generalization gap than the trend suggests, given its fine-tuning distance.
The deviation from the trend suggests that anchoring the networks parameters to an interpretable algorithm improves performance, a finding supported by \cref{fig:gen-gap-anchor}, which analyzes the effect of anchoring a randomly initialized model to the scattering network's parameters.

\section{Conclusion}

We introduced a framework for \emph{specific algorithmic interpretability}
that ties a particular neural-network parameter realization to an explicit
algorithm rather than merely an algorithm-inspired architecture. The key idea is to preserve this algorithmic reference during learning through controlled fine-tuning, whose distance measures the deviation from the disentangler and allows controlling the generalization bound.

\emph{Significance.}
Our framework provides a way to train
highly over-parameterized networks in low-data regimes without overfitting and without requiring a
network pretrained on a large external dataset. In contrast to standard transfer learning, where the useful initialization is typically obtained by pretraining on vast amounts of data, here it is supplied directly by an explicit interpretable algorithm. Additionally, the framework suggests a new role for classical algorithms in the deep-learning era: rather than being
discarded in favor of flexible learned models, they can serve as algorithmically interpretable
initializations that combine explicit domain knowledge with data-driven
adaptation. 

\emph{Limitations and Future Work.} The present framework is mainly suited to classification problems.  Future work may extend this
perspective to other settings, including generative models, and instantiate
the framework for other data modalities.
Such extensions may clarify when explicit algorithmic structure provides a
useful foundation for modern learned models.

\paragraph{Acknowledgement. }
AvS and GL were partially supported by NSF award DMS-2427955. 
RL was supported by a grant from the United States-Israel Binational Science Foundation (BSF), Jerusalem, Israel, and the United States National Science Foundation (NSF), (NSF-BSF, grant No. 2024660), and by the Israel Science Foundation (ISF grant No. 1937/23).

\bibliographystyle{plainnat}
\bibliography{FineTuning}

\appendix

\clearpage
\startcontents[appendices]

\section*{Appendix}
\printcontents[appendices]{}{1}{}

\section{Background and Preliminaries}

\subsection{Additional related work on interpretability}
\label{app:related_work}

Interpretability methods can be divided into intrinsic approaches, which design models whose structure is intended to be understandable by construction, and post-hoc approaches, which analyze an already trained model
\citep{Rudin2018StopEB,Arrieta2019ExplainableAI}. Our work belongs to the
intrinsic setting and, more specifically, to the literature connecting neural
networks with explicit computational algorithms
\citep{Monga2021Unrolling,Gregor2010LearningFA,Papyan2017CNN,
Yaodong2024WhiteBoxTransformers}. As discussed in the main text, our focus is
on the algorithmic meaning of a particular learned parameter realization,
rather than only the algorithmic structure shared by a parameterized family.

\paragraph{Concept and prototype-based models.}
Some intrinsic approaches design neural networks whose decisions pass through
human-interpretable intermediate representations. Concept bottleneck models
predict predefined semantic concepts before producing the final prediction
\citep{Koh2020ConceptBottleneck}, while prototype-based networks classify examples
through similarity to learned representative parts
\citep{Chen2019Looks}. These approaches make the decision process interpretable
in terms of concepts or examples; our work instead seeks to associate the
network computation itself with an explicit algorithm.

\paragraph{Algorithmic alignment.}
Our work is also related to algorithmic alignment, which studies the
correspondence between the computational structure of a neural network and the
decomposition of a target algorithm \citep{xu2020what}. When such a
correspondence exists, network modules can be associated with computational
subroutines of the target algorithm, which may facilitate learning and
generalization. Algorithmic alignment primarily concerns whether
the organization of the network is compatible with the decomposition of an
algorithmic computation. 
Such structural correspondence does not by itself guarantee an algorithmic interpretation of the fully trained network, since the computation it implements is determined jointly by its architecture and learned parameters.
Our focus is complementary: we ask whether the particular parameter realization
of the trained network retains an explicit algorithmic meaning. 

\paragraph{Mechanistic interpretability.}
Mechanistic interpretability seeks to reverse-engineer trained neural networks
by identifying internal features, circuits, and computational mechanisms that
give rise to their behavior \citep{Olah2020ZoomIA}. It is closely related to
our objective in that both aim to understand the computation performed by a
particular trained model rather than only characterize the architecture at a
generic level. The main difference
is methodological. Mechanistic interpretability typically begins with an
already trained model and attempts to recover its internal computation
post hoc. Our approach instead begins from a parameter realization that
implements a known algorithm and constrains subsequent learning so that the
trained model remains related to this interpretable reference point.

\subsection{Wavelet and Scattering Preliminaries}

We review here the wavelet-based statistical construction used in our
texture experiments, namely, the wavelet scattering transform.
The scattering transform constructs a hierarchy of representations by
cascading predefined wavelet convolutions with pointwise modulus nonlinearities
along a collection of scattering paths, followed by spatial averaging.
The resulting scattering coefficients form a finite family of statistics. Under the stationary ergodic model introduced in \cref{subsec:texture_model}, these spatial averages provide estimators of statistics of the
underlying generating process.

\subsubsection{Oriented wavelet transform}
\label{app:wavelet_transform}

We define the oriented wavelet transform following
\citet{mallat2012groupinvariantscattering}, specialized to two-dimensional images and
dyadic scales.

\begin{definition}[Wavelet Transform]
    \label{def:wavelet-transform}
    Let \( \psi \in L^2(\R^2) \) be a complex function that we call a \emph{mother wavelet}, and let \( \phi \in L^2(\R^2) \) be a complex function with \( \int\phi=1 \), which we call a low-pass function.
    For an integer \( L\geq1 \), let \( G=\{r_k\}_{k=0}^{2L-1} \subset SO(2)\) be the rotations by the angles \( \theta_k = k\pi/L \).
    A dilation by \( 2^j \), \( j\in\Z \), and rotation by \( r \in G \) of \( \psi \) produce
    \( \psi_{\lambda} (u) = 2^{-2j} \psi(2^{-j} r^{-1} u) \), where we write \( \lambda = 2^j r \).
    For \( J\in\Z \), the low-pass filter at scale \( 2^J \) is \( \phi_J(u) = 2^{-2J} \phi(2^{-J} u) \); since \( \int\phi_J=1 \), convolution with \( \phi_J \) performs a weighted spatial average over a region of size proportional to \( 2^J \).
    The wavelet transform at scale \( 2^J \) of a signal \( x \in L^2(\R^2) \) is the collection
    \[
    \mathcal{W}_J x = \{ x * \phi_J, ( x * \psi_{\lambda} )_{\lambda \in \Lambda_J} \},
    \qquad
    \Lambda_J  = \{ 2^j r : r \in G,\ j<J \},
    \]
    with norm
    \[
    \| \mathcal{W}_J x \|^2 = \| x * \phi_J \|^2 + \sum_{\lambda \in \Lambda_J} \| x * \psi_{\lambda} \|^2 .
    \]
\end{definition}

The coefficient $x*\psi_\lambda(u)$ is complex. Its modulus measures the local
response of $x$ to structures at the scale $2^j$ and orientation indexed by
$\lambda$, while its phase encodes the precise position of the oscillation.
The scattering transform below retains the modulus, which varies slowly in
space and can therefore be averaged stably, and recovers part of the
information discarded with the phase through a cascade of further wavelet
transforms.

\paragraph{Symmetry for real-valued inputs.}
For Morlet wavelets, $\psi(-u)=\overline{\psi(u)}$. Since $r_{k+L}=-r_k$, this
gives $\psi_{2^jr_{k+L}}=\overline{\psi_{2^jr_k}}$, so for real-valued $x$
we have $|x*\psi_{2^jr_{k+L}}|=|x*\psi_{2^jr_k}|$. All inputs to the
scattering cascade below are real-valued (the image itself and the moduli
computed from it), so in practice only the $L$ orientations $\theta_k$,
$k=0,\ldots,L-1$, in $[0,\pi)$ are computed.

\paragraph{Frame bounds.}
Recall that a family $\{f_k\}_{k\in\Gamma}$ in a Hilbert space
$\mathcal H$ is a \emph{frame} with bounds $0<A\leq B$ if
$A\|f\|^2\leq\sum_{k\in\Gamma}|\langle f,f_k\rangle|^2\leq B\|f\|^2$ for all
$f\in\mathcal H$. For the wavelet transform, frame bounds are guaranteed by
the Littlewood--Paley condition.

\begin{proposition}[Littlewood--Paley condition; {\citealt{mallat2012groupinvariantscattering}}]
    \label{prop:littlewood-paley}
    If there exists \( \epsilon\in[0,1) \) such that, for almost all \( \omega \in \R^2 \),
    \[
    1 - \epsilon \leq | \hat{\phi}(2^J \omega)|^2 + \frac{1}{2} \sum_{j < J} \sum_{r \in G} | \hat{\psi} (2^j r^{-1} \omega ) |^2 \leq 1,
    \]
    then \( \mathcal{W}_J \) is a frame with bounds \( 1-\epsilon \) and \( 1\) on real-valued signals:
    \[
    (1- \epsilon) \|x \|^2 \leq \| \mathcal{W}_J x \|^2 \leq \| x \|^2, \qquad x \in L^2(\R^2)\ \text{real-valued}.
    \]
\end{proposition}

The lower bound ensures that $\mathcal W_J$ retains all the information in
$x$, and the upper bound makes $\mathcal W_J$ non-expansive. Since the complex
modulus is also non-expansive, the scattering transform defined below inherits
this property (Appendix~\ref{app:scattering_parameter_stability}).

\paragraph{Discrete setting.}
In practice, images are defined on a finite grid, all convolutions are
discrete, and the Littlewood--Paley condition is imposed on the discrete
frequency grid. The wavelet transform is computed at finitely many scales
$0\leq j<J$, where $j=0$ is the finest scale resolved by the grid.
With a slight abuse of notation, we write
\[
\Lambda_J=\{2^jr_k:\ 0\leq j<J,\ 0\leq k<L\}
\]
for this finite index set of computed wavelets.

\subsubsection{Wavelet scattering transform}
\label[appendix]{app:scattering_transform}
The wavelet scattering transform constructs nonlinear multiscale statistics by
cascading wavelet convolutions and complex modulus nonlinearities
\citep{mallat2012groupinvariantscattering}.
We define it formally as follows.

\begin{definition}[Windowed Scattering Transform]
    \label{def:scattering-transform}
    For \( \lambda\in\Lambda_J \), define the scattering operator \( U[\lambda]x := |x * \psi_\lambda| \).
    A sequence \(p = ( \lambda_1, \ldots, \lambda_q )\in\Lambda_J^q \) is called a path of length \( q\geq0 \); the empty path ($q=0$) is denoted \( \emptyset \).
    The scattering propagator along \( p \) is defined by \( U [\emptyset] x := x \) and
    \[
    U [p] x := U [\lambda_q] \cdots U [\lambda_1] x
    = \bigl| U[(\lambda_1,\ldots,\lambda_{q-1})]x * \psi_{\lambda_q} \bigr| .
    \]
    The windowed scattering coefficient associated with \( p \) is
    \begin{equation}\label{eq:windowed_scattering}
        S_J [p] x := U [p] x * \phi_J ,
    \end{equation}
    and, for a set of paths \( \mathcal P \), the windowed scattering transform is the collection
    \( S_J(x) := \bigl( S_J[p]x \bigr)_{p \in \mathcal P} \).
\end{definition}

We use the standard choice of \emph{frequency-decreasing} paths of length at
most $q_{\max}$,
\[
\mathcal P_J
:=
\bigl\{
(2^{j_1}r_{k_1},\ldots,2^{j_q}r_{k_q})\in\Lambda_J^q:\
0\leq q\leq q_{\max},\
j_1<j_2<\cdots<j_q
\bigr\}.
\]
Since $U[\lambda]x$ has most of its energy at frequencies below the central
frequency of $\psi_\lambda$, paths along which the scale does not increase carry
little energy and are discarded \citep{bruna_invariant_2013}. In the discrete
setting, $\mathcal P_J$ is finite. Unless stated otherwise,
$\mathcal P=\mathcal P_J$ and $S_J(x)=(S_J[p]x)_{p\in\mathcal P_J}$.

Under the Littlewood--Paley condition, the windowed scattering transform over
all paths is non-expansive
\citep[Proposition~2.5]{mallat2012groupinvariantscattering},
\citep[Proposition~8.3]{grohs_scattering_2022}. Hence, so is its restriction to
any subset of paths, such as $\mathcal P_J$:
\[
\|S_J(x)-S_J(x')\|\leq\|x-x'\|.
\]

\subsubsection{Scattering coefficients as statistics of an ergodic process.}
The connection with the stochastic model in Section \ref{subsec:stochastic_model} is immediate.
Each scattering path defines a particular function of a random field. For a
path $p$, define
\begin{equation}
\label{eq:scattering_observable}
f_p(X)
:=
U[p]X(0).
\end{equation}
If $X$ is stationary with law $Q$, stationarity implies that
$
\mathbb E_Q[U[p]X(u)]
$
does not depend on $u$. The corresponding statistic of the process is therefore
\begin{equation}
\label{eq:expected_scattering}
\overline S_Q[p]
:=
\mathbb E_Q[U[p]X(0)]
=
\mathbb E_Q[f_p(X)].
\end{equation}
The collection $\{f_p:p\in\mathcal P\}$ is thus precisely a particular finite
choice of the functions used to define statistics in the ergodic model of the
main text.

The coefficient $S_J[p]x(u)$ in
\cref{eq:windowed_scattering} is a weighted spatial average of this function
over translations of the observed realization. Indeed, using the
translation action defined in the main text, it can be viewed as a finite
weighted approximation to the ergodic average of $f_p$. If $\phi_J$ has
sufficiently wide support, then for an ergodic process this spatial average
approximates
$
\overline S_Q[p].
$
Hence, the scattering network computes from a single finite texture image
estimates of a prescribed collection of statistics of its underlying
generating process.

Mallat introduced this stochastic interpretation explicitly through the
\emph{expected scattering transform}, whose coefficients are
$\mathbb E_Q[U[p]X]$ \citep{mallat2012groupinvariantscattering}. Unlike the
power spectrum, which depends only on second-order information, scattering
coefficients capture nonlinear higher-order dependencies through the cascade
of wavelet transforms and modulus operations. The predefined scattering
network therefore specifies both the functions whose expectations are measured
and the algorithm used to estimate them from a finite realization.

\section{Theoretical Analysis and Proofs}
\subsection{PAC-Bayesian Analysis Around a Data-Independent Initialization}
\label{app:general_ft_bound}
We formalize the statistical role of restricting fine-tuning around
$\Theta_0$ using a PAC-Bayesian analysis. Since $\Theta_0$ is specified by the predefined disentangling
algorithm and is independent of the training sample, it can serve as a
data-independent reference point in a PAC-Bayesian analysis. We consider
predictors obtained by fine-tuning within a neighborhood of this realization.
For $R>0$, define
$$
B_R(\Theta_0)
:=
\left\{
\Theta:
\norm{\Theta-\Theta_0}_2\le R
\right\}.
$$
Under a suitable local stability condition on parameter perturbations, we
derive a generalization bound that holds uniformly over
$B_R(\Theta_0)$. Crucially, the bound depends on the distance from the
reference initialization, but not on the predictive performance of
$F_{\Theta_0}$ itself.

\subsubsection{Margin risks and PAC-Bayesian margin bound}
Let $\mathcal D$ denote the distribution of labeled pairs
$(x,y)\in\mathcal X\times[C]$, whose marginal on $\mathcal X$ is the data
distribution $P$ in \cref{eq:mixture_processes}, and let
$\{(x_n,y_n)\}_{n=1}^M$ be a training set of $M$ i.i.d.\ samples from
$\mathcal D$. For a predictor $F_\Theta:\mathcal X\to\R^C$, recall the
classification margin
$
m_\Theta(x,y)=F_\Theta(x)[y]-\max_{j\neq y}F_\Theta(x)[j].
$
For $\gamma\geq0$, the statistical and empirical $\gamma$-margin risks are
\[
R_\gamma(F_\Theta)
:=
\Pbb_{(x,y)\sim\mathcal D}\bigl[m_\Theta(x,y)\leq\gamma\bigr],
\qquad
\widehat R_\gamma(F_\Theta)
:=
\frac1M\sum_{n=1}^M\mathbbm 1\bigl[m_\Theta(x_n,y_n)\leq\gamma\bigr].
\]
In particular, $R_0(F_\Theta)$ is the statistical classification risk, with
ties counted as errors.
We begin from the following standard PAC-Bayesian margin bound.
\begin{lemma}[{\citealt[Lemma 1]{neyshabur_pac-bayesian_2018}}]
\label{lemma:pac}
Let $F_{\Theta}:\mathcal{X}\to \mathbb{R}^C$
be any predictor, parameterized by \(\Theta\). Let \(\pi\) be any prior distribution over the parameter space that is independent of the training data. Then, for any \(\gamma>0,\delta\in(0,1)\), with probability at least \(1-\delta\) over the draw of a training set of size \(M\), the following holds.
For any \(\Theta\) and any random perturbation \(U\) such that
$\Pbb_U\left[\sup_{x\in\mathcal{X}}\norm{F_{\Theta+U}(x)-f_{\Theta}(x)}_\infty<\frac{\gamma}{4}\right]\ge\frac12,$
we have
\[R_{0}(F_{\Theta})\le\widehat R_{\gamma}(F_{\Theta})+4\sqrt{\frac{\KL(\mathcal Q\,\|\,\pi)+\ln\frac{6M}{\delta}}{M-1}},
\]
where $\mathcal Q=Law(\Theta+U)$ denotes the posterior distribution.
\end{lemma}

Lemma~\ref{lemma:pac} requires the predictor to be sufficiently insensitive to a random perturbation of its parameters. To verify this condition uniformly for predictors in a neighborhood of the algorithmically interpretable realization $\Theta_0$, we introduce a deterministic local stability condition under bounded parameter perturbations.
\subsubsection{Local perturbation stability}
\begin{definition}[Local Perturbation Stability]\label{def:local_stability}
    Let $\F=\{F_\Theta: \X\rightarrow\R^C\}_{\Theta\in\W}$ be a parametrized hypothesis family, where $\Theta=\operatorname{vec}\bigl(\{W_i\}_{i=1}^d\bigr).$ For $\Theta_0\in \W$, $R > 0$, and $\tau > 0$, define 
    $$B_R(\Theta_0):=\{\Theta\in\W \mid \norm{\Theta-\Theta_0}_2 \leq R\}\},$$
    and 
    $$\U_\tau:=\{U=\operatorname{vec}\bigl(\{U_i\}_{i=1}^d\bigr): \max_{i\in[d]}\norm{U_i}_2\le\tau\}.$$
    We say that the family $\F$ is ($R, \tau, \epsilon$)-stable under bounded parameter perturbations around $\Theta_0$ if, for every $\Theta\in B_R(\Theta_0)$ and every $U\in\U_\tau$,
    $$\sup_{x\in\X} \norm{ F_{\Theta+U}(x) - F_{\Theta}(x) }_{\infty} < \epsilon.$$
\end{definition}
The stability condition allows the perturbation requirement in Lemma~\ref{lemma:pac} to be verified uniformly over a neighborhood of the data-independent realization $\Theta_0$. This leads to the following bound.

\subsubsection{Uniform generalization bound}

\begin{restatable}[Uniform Generalization Bound Around a Data-Independent Initialization]{corollary}{}
    Let $\F$ be a parametrized hypothesis family, and let $\Theta_0=vec(\{W_i\}_{i=1}^d)$ be an initialization independent of the training data, where $W_i\in\R^{m_i\times n_i}$, and define $h:=\max_{i\in[d]}\{m_i, n_i\}.$ Fix $R,\gamma,\tau>0$, and suppose that $\F$ is ($R,\tau,\frac{\gamma}{4}$)-stable under bounded parameter perturbations around $\Theta_0$. 
    Then, for any $\delta \in (0,1)$, with probability at least $1-\delta$ over the draw of a training set of size $M$,
    the following holds simultaneously for all $\Theta\in B_R(\Theta_0)$:
    \begin{equation}
    R_{0} (F_{\Theta}) \leq \widehat R_{\gamma}(F_{\Theta}) + 4 \sqrt{ \left( \frac{R^2 h \ln(4dh)}{\tau^2} + \ln \frac{6 M}{\delta } \right) / \left( M - 1 \right)} \label{eq:general_ft_bound}    
    \end{equation}
\end{restatable}
\begin{proof}
Set
$
\sigma=\frac{\tau}{\sqrt{2h\ln(4dh)}}
$
and choose the data-independent prior
$
\pi=\mathcal N(\Theta_0,\sigma^2 I).
$
For any $\Theta\in B_R(\Theta_0)$, define the posterior
$
\mathcal Q_\Theta=\mathcal N(\Theta,\sigma^2 I).
$
A draw from $\mathcal Q_\Theta$ can be written as $\Theta+U$, where
$U\sim\mathcal N(0,\sigma^2I)$. Write
$
U=\operatorname{vec}(\{U_i\}_{i=1}^d),
$
where $U_i\in\R^{m_i\times n_i}$.

By the Gaussian spectral-norm bound
\citep[Eq.~4.6]{Tropp_2011}, together with $m_i,n_i\leq h$, for every
$i\in[d]$ and $t\geq0$,
$$
\Pbb\!\left(\|U_i\|_2>t\right)
\leq
2h\exp\!\left(-\frac{t^2}{2h\sigma^2}\right).
$$
Applying a union bound with $t=\tau$ gives
$$
\Pbb\!\left(
\max_{i\in[d]}\|U_i\|_2\leq\tau
\right)
\geq
1-
2dh\exp\!\left(-\frac{\tau^2}{2h\sigma^2}\right).
$$
By the choice of $\sigma$,
$$
2dh\exp\!\left(-\frac{\tau^2}{2h\sigma^2}\right)
=
\frac12,
$$
and hence
$$
\Pbb\!\left(
\max_{i\in[d]}\|U_i\|_2\leq\tau
\right)
\geq
\frac12.
$$

Since $\F$ is
$(R,\tau,\frac{\gamma}{4})$-stable around $\Theta_0$, for every
$\Theta\in B_R(\Theta_0)$, the event
$
\max_{i\in[d]}\|U_i\|_2\leq\tau
$
implies
$$
\sup_{x\in\X}
\|F_{\Theta+U}(x)-F_\Theta(x)\|_\infty
<
\frac{\gamma}{4}.
$$
Therefore,
$$
\Pbb_U\!\left[
\sup_{x\in\X}
\|F_{\Theta+U}(x)-F_\Theta(x)\|_\infty
<
\frac{\gamma}{4}
\right]
\geq
\frac12,
$$
so the perturbation condition in Lemma~\ref{lemma:pac} is satisfied.

Because $\pi$ and $\mathcal Q_\Theta$ are Gaussian distributions with the
same covariance,
$$
\KL(\mathcal Q_\Theta\|\pi)
=
\frac{\|\Theta-\Theta_0\|_2^2}{2\sigma^2}.
$$
Since $\Theta\in B_R(\Theta_0)$,
$$
\KL(\mathcal Q_\Theta\|\pi)
\leq
\frac{R^2}{2\sigma^2}
=
\frac{R^2h\ln(4dh)}{\tau^2}.
$$
Substituting this bound into Lemma~\ref{lemma:pac} yields
$$
R_0(F_\Theta)
\leq
\widehat R_\gamma(F_\Theta)
+
4
\sqrt{
\frac{
\frac{R^2h\ln(4dh)}{\tau^2}
+
\ln\frac{6M}{\delta}
}{
M-1
}
}.
$$
On the event of Lemma~\ref{lemma:pac}, the bound holds for every $\Theta$
whose perturbation satisfies the required condition. Since the above argument
verifies this condition for every $\Theta\in B_R(\Theta_0)$, the bound holds
simultaneously for all $\Theta\in B_R(\Theta_0)$.
\end{proof}

\subsection{Local Perturbation Stability of Scattering}
\label{app:scattering_parameter_stability}

We verify here that the predefined scattering realization satisfies the local
parameter-stability condition used in
Corollary~\ref{coro:general_ft_bound}. We first establish local parameter stability of the predefined scattering
feature extractor. We then show that this stability extends to the full
predictor obtained by composing the scattering representation with a linear
classifier, as used in our experiments. Let $\mathcal P$ be a collection of
$P$ scattering paths, each of length at most $L$, and write
\[
\Theta_0
=
\operatorname{vec}
\bigl(
\{W_i^{(p)}\}_{p\in\mathcal P,\,i\in[d_p]}
\bigr)
\]
for the predefined scattering parameters, where $d_p$ denotes the length of
path $p$. We assume that the predefined wavelet operators and the final
low-pass operator are non-expansive,
\[
\|W_i^{(p)}\|_2\leq1,
\qquad
\|\Phi\|_2\leq1.
\]
Let $S_\Theta$ denote the corresponding scattering network obtained by
parameterizing the wavelet operators by $\Theta$, while keeping the modulus
and final low-pass operations fixed.

\begin{lemma}[Local perturbation stability of scattering]
\label{lemma:scattering_parameter_stability}
Let
$
\mathcal X=\{x:\|x\|_2\leq B\}.
$
Then, for any $R,\tau\geq0$, the scattering family $\{S_\Theta\}$ is
$(R,\tau,\epsilon)$-stable around $\Theta_0$ in the sense of
Definition~\ref{def:local_stability}, with
\[
\epsilon
=
\sqrt{P}\,B L\tau(1+R+\tau)^{L-1}.
\]
In particular, for any $\gamma>0$, let
$$
a_\gamma
:=
\frac{\gamma}{4\sqrt{P}BL},
\qquad
\tau_\gamma
:=
\frac{a_\gamma}{(1+R+a_\gamma)^{L-1}}.
$$
Then $\{S_\Theta\}$ is
$(R,\tau_\gamma,\frac{\gamma}{4})$-stable under bounded parameter
perturbations around $\Theta_0$.
\end{lemma}

\begin{proof}
Fix $\Theta\in B_R(\Theta_0)$ and
$U\in\mathcal U_\tau$. For each path $p\in\mathcal P$, write
\[
\Theta^{(p)}
=
\{V_i^{(p)}\}_{i=1}^{d_p},
\qquad
U^{(p)}
=
\{U_i^{(p)}\}_{i=1}^{d_p}.
\]
Since $\|\Theta-\Theta_0\|_2\leq R$, each parameter block satisfies
\[
\|V_i^{(p)}-W_i^{(p)}\|_2
\leq
\|V_i^{(p)}-W_i^{(p)}\|_F
\leq R.
\]
Using $\|W_i^{(p)}\|_2\leq1$, we therefore obtain
\begin{equation}
\|V_i^{(p)}\|_2\leq1+R,
\qquad
\|V_i^{(p)}+U_i^{(p)}\|_2\leq1+R+\tau.
\label{eq:scattering_filter_norm_bound}
\end{equation}

Consider a fixed path $p$ of length $d_p$. Define
\[
z_0=\widetilde z_0=x,
\]
and recursively
\[
z_i
=
\rho(V_i^{(p)}z_{i-1}),
\qquad
\widetilde z_i
=
\rho((V_i^{(p)}+U_i^{(p)})\widetilde z_{i-1}),
\]
where $\rho(z)=|z|$ is applied entrywise. Since the complex modulus is
non-expansive,
\begin{align}
\|\widetilde z_i-z_i\|_2
&\leq
\|(V_i^{(p)}+U_i^{(p)})\widetilde z_{i-1}
-
V_i^{(p)}z_{i-1}\|_2
\nonumber\\
&\leq
\|V_i^{(p)}+U_i^{(p)}\|_2
\|\widetilde z_{i-1}-z_{i-1}\|_2
+
\|U_i^{(p)}\|_2\|z_{i-1}\|_2.
\label{eq:scattering_perturbation_recursion}
\end{align}

Moreover, by \eqref{eq:scattering_filter_norm_bound},
\[
\|z_{i-1}\|_2
\leq
(1+R)^{i-1}\|x\|_2.
\]
Hence
\[
\|\widetilde z_i-z_i\|_2
\leq
(1+R+\tau)\|\widetilde z_{i-1}-z_{i-1}\|_2
+
\tau(1+R)^{i-1}\|x\|_2.
\]
Since $\widetilde z_0=z_0$, iterating this inequality yields
\begin{align}
\|\widetilde z_{d_p}-z_{d_p}\|_2
&\leq
\tau\|x\|_2
\sum_{j=0}^{d_p-1}
(1+R+\tau)^{d_p-1-j}(1+R)^j
\nonumber\\
&\leq
d_p\tau(1+R+\tau)^{d_p-1}\|x\|_2.
\label{eq:single_path_scattering_perturbation}
\end{align}

The final low-pass operator is non-expansive, so
\[
\|
S_{\Theta+U}[p](x)-S_\Theta[p](x)
\|_2
\leq
d_p\tau(1+R+\tau)^{d_p-1}\|x\|_2.
\]
Using $d_p\leq L$ and $\|x\|_2\leq B$ gives
\[
\|
S_{\Theta+U}[p](x)-S_\Theta[p](x)
\|_2
\leq
B L\tau(1+R+\tau)^{L-1}.
\]

Finally, stacking the outputs of the $P$ scattering paths gives
\begin{align}
\|
S_{\Theta+U}(x)-S_\Theta(x)
\|_\infty
&\leq
\|
S_{\Theta+U}(x)-S_\Theta(x)
\|_2
\nonumber\\
&\leq
\left(
\sum_{p\in\mathcal P}
\|
S_{\Theta+U}[p](x)-S_\Theta[p](x)
\|_2^2
\right)^{1/2}
\nonumber\\
&\leq
\sqrt{P}\,B L\tau(1+R+\tau)^{L-1}.
\end{align}
This bound holds uniformly over
$x\in\mathcal X$, $\Theta\in B_R(\Theta_0)$, and
$U\in\mathcal U_\tau$, proving the result.
It remains to verify the final claim. Since
$
(1+R+a_\gamma)^{L-1}\geq1,
$
we have $\tau_\gamma\leq a_\gamma$. Therefore,
$$
1+R+\tau_\gamma
\leq
1+R+a_\gamma.
$$
Substituting the definition of $\tau_\gamma$ into the perturbation bound gives
\begin{align}
\epsilon
&=
\sqrt{P}\,BL\tau_\gamma
(1+R+\tau_\gamma)^{L-1}
\nonumber\\
&\leq
\sqrt{P}\,BL
\frac{a_\gamma}{(1+R+a_\gamma)^{L-1}}
(1+R+a_\gamma)^{L-1}
\nonumber\\
&=
\sqrt{P}\,BL a_\gamma
=
\frac{\gamma}{4}.
\end{align}
\end{proof}

\begin{corollary}[Stability with a linear classifier]
\label{coro:scattering_linear_stability}
Let $F_{\Theta,A}(x)=A S_\Theta(x),$
and let $(\Theta_0,A_0)$ denote the initialization of the full predictor.
Assume that $A_0$ is chosen independently of the training sample. Then the
family $\{F_{\Theta,A}\}$ is locally parameter-stable around
$(\Theta_0,A_0)$.

More precisely, for
$
\|(\Theta,A)-(\Theta_0,A_0)\|_2\leq R
$
and perturbations
$
\|(U,V)\|_2\leq\tau,
$
we have
\[
\|F_{\Theta+U,A+V}(x)-F_{\Theta,A}(x)\|_\infty
\leq
(\|A_0\|_2+R)\epsilon_S
+
\tau\sqrt{P}\,B(1+R+\tau)^L,
\]
where
\[
\epsilon_S
=
\sqrt{P}\,BL\tau(1+R+\tau)^{L-1}
\]
is the perturbation bound of
Lemma~\ref{lemma:scattering_parameter_stability}.
\end{corollary}

\begin{proof}
Using
\[
F_{\Theta+U,A+V}(x)-F_{\Theta,A}(x)
=
A\bigl(S_{\Theta+U}(x)-S_\Theta(x)\bigr)
+
V S_{\Theta+U}(x),
\]
we obtain
\[
\|F_{\Theta+U,A+V}(x)-F_{\Theta,A}(x)\|_\infty
\leq
\|A\|_2
\|S_{\Theta+U}(x)-S_\Theta(x)\|_2
+
\|V\|_2\|S_{\Theta+U}(x)\|_2.
\]
Since
$
\|A\|_2\leq\|A_0\|_2+R
$
and
\[
\|S_{\Theta+U}(x)\|_2
\leq
\sqrt{P}\,B(1+R+\tau)^L,
\]
the result follows from
Lemma~\ref{lemma:scattering_parameter_stability}.
\end{proof}

\section{Experimental Details and Results}

\subsection{Texture Synthesis from Scattering Statistics}
\label{app:scattering_synthesis}

We use the microcanonical gradient-descent approach of
\citet{bruna2018Microcanonical} to synthesize textures whose prescribed
statistics match those of a reference image. The purpose of this experiment is
to qualitatively assess how much perceptually relevant texture information is
retained by the scattering statistics used in our disentangler.

Let $x$ be a reference texture. We compute the
two-dimensional scattering transform and spatially average each scattering
coefficient. Denoting this empirical statistic vector by
$$
\widehat S_J(x)
=
\operatorname{Avg}_u S_J(x)(u),
$$
texture synthesis seeks an image $z$ whose statistics are close to those of
$x$. This spatial averaging is the finite-image counterpart of the process
statistics discussed in Section~\ref{subsec:texture_model}. In the
implementation, the averaging is performed directly over the spatial
coordinates of the scattering coefficients.

For grayscale images, we minimize
\begin{equation}
\label{eq:scattering_synthesis_objective}
\mathcal L_{\mathrm{scat}}(z)
=
\frac{1}{D}
\left\|
\widehat S_J(z)-\widehat S_J(x)
\right\|_2^2,
\end{equation}
where $D$ is the dimension of the scattering statistic vector. For RGB
images, we additionally preserve the correlations between color channels.
Let $C(x)\in\R^{3\times3}$ denote the correlation matrix of the centered and
standardized RGB channels. We then optimize
\begin{equation}
\label{eq:scattering_synthesis_rgb_objective}
\mathcal L_{\mathrm{syn}}(z)
=
\mathcal L_{\mathrm{scat}}(z)
+
\left\|
C(z)-C(x)
\right\|_F^2.
\end{equation}
Thus, the synthesis procedure constrains the generated image through the
selected scattering statistics, together with a simple color constraint for
RGB textures. This is the objective implemented in our synthesis procedure.

The optimization is initialized from an independent random image. More
precisely, each channel is initialized with Gaussian noise having the same
mean and standard deviation as the corresponding channel of the reference
image. We optimize the image pixels using
L-BFGS with a strong-Wolfe line search while keeping the scattering transform
fixed. The scattering coefficients are
computed using the two-dimensional scattering implementation in
Kymatio~\citep{kymatio}.

After optimization, we apply channel-wise histogram matching between the
synthesized and reference images for visualization. This step matches the
marginal intensity distributions of the synthesized image to those of the
reference and is not part of the optimization objective above.

Because the optimization starts from an independent random image, similarity
between the synthesized and reference textures cannot be attributed to copying
their spatial arrangements. Rather, it indicates that matching the selected
statistics is sufficient to recover characteristic perceptual structures of
the reference texture. We therefore use these synthesis results as qualitative
evidence that the chosen scattering statistics retain meaningful texture
information.

\subsection{Texture Classification Setup}
\subsubsection{Datasets and Sampling}
\label{app:datasets_split}
To evaluate sample efficiency, we train on class-stratified subsets ranging
from the full training set down to $0.2\%$ for MinC-2500, $0.5\%$ for
EuroSAT, $10\%$ for DTD, and \( 2\%\) for SuroTex.
DTD has ten official train/val/test partitions of its 5,640 images (47 classes with 120 examples each); each partition allocates 40 images per class to each split; we merge the the offiical train and val halves of a partition into a single 3,760 image pool and re-split it with an 80/20 train/val split, using class-stratified splitting.
The official 1,880 image test split is not modified.
For DTD, we evaluate five of the
ten standard partitions to balance computational cost with robustness across
dataset splits.
The remaining three datasets have no official train/val/test split, so we construct splits as follows.
First, we separate a class-stratified test split, taking 15\% of the data for MinC-2500 and 20\% for SuroTex and EuroSat.
The remaining data is split 80/20 into training and validation sets.
Subsampling is applied only to the training dataset, so the validation and test sets are identical across training dataset sizes.
At the smallest training size, the model is trained on 4 images per class for MinC-2500, 8-9 for EuroSAT, 6 for DTD, and a single image per class for SuroTex.
Validation accuracy is used for checkpoint selection and early stopping, and all reported numbers from from the held out test set.

\subsubsection{Model Architectures and Baselines}
\label{app:architectures}
We next describe the model architectures used in our experiments, including
our trainable scattering network, the three baselines, and the common
classification head, followed by a discussion of Parametric Scattering
Networks \citep{gauthier_parametric_2022}.

\paragraph{Our trainable scattering architecture.}
We describe the convolutional implementation used in our experiments to realize
the predefined scattering transform as a trainable neural network. The
construction is designed so that, at $\Theta=\Theta_0$, the network implements
the scattering transform exactly, while its wavelet operators are represented
as spatial convolutional parameters that can subsequently be fine-tuned.

\emph{Spatial implementation of scattering.}
Standard implementations such as Kymatio~\citep{kymatio} compute scattering
coefficients efficiently using fixed wavelet filters in the Fourier domain.
For our purpose, we instead require the wavelet operations to appear explicitly
as trainable convolutional layers. We therefore implement scattering in the
spatial domain using steerable Morlet wavelets
\citep{Steerable,bruna_invariant_2013} together with the fast multiscale
wavelet transform~\citep{WaveletTourBook}.

Let $\{g_k\}_{k=0}^{L}$ denote the one-step filter bank at the finest scale
$j=0$. Here $g_0$ is the low-pass filter, while $g_1,\ldots,g_L$ are band-pass
Morlet filters, $g_k$ having orientation $\theta_{k-1}=(k-1)\pi/L$ as in
Definition~\ref{def:wavelet-transform}. For an input $x$, define the
scattering block
$
Tx(u,k)=(x* g_k)(2u),
$
so that each step consists of filtering followed by subsampling by a factor of
two.

Repeated application of this operation implements the multiscale transform.
Because the input is subsampled after each step, applying the same finest-scale
filters at the next resolution corresponds to filtering at a coarser scale in
the coordinates of the original image. This avoids explicitly constructing
larger spatial filters for successive dyadic scales. In this sense, the
multiscale filters of Definition~\ref{def:wavelet-transform} are the
equivalent filters of the cascade: $j$ low-pass steps with $g_0$ followed by
$g_k$ realize the wavelet $\psi_{2^jr_{k-1}}$, and $J$ low-pass steps realize
the low-pass filter $\phi_J$, for $0\leq j<J$.

\emph{Scattering cascade.}
To obtain the scattering transform, the fast wavelet cascade is propagated
along the band-pass branches. For an intermediate representation $z$, the
low-pass response
$
(z* g_0)(2u)
$
produces an averaged component, while each band-pass response is followed by
the complex modulus,
$$
(z* g_k)(2u)
\longmapsto
|(z* g_k)(2u)|,
\qquad k=1,\ldots,L,
$$
before being passed to subsequent blocks.
Following standard neural network architectures, we include an optional batch normalization layer between blocks, immediately following the complex modulus.
The inclusion of the batch normalization layer was part of the architecture sweep and is reported in \cref{tab:arch_config_deep_scatter}.

A path $p=(\lambda_1,\ldots,\lambda_q)\in\mathcal P_J$ with
$\lambda_i=2^{j_i}r_{k_i}$ is realized by applying the band-pass filter
$g_{k_i+1}$, followed by the modulus, at block $j_i+1$, and the low-pass
filter $g_0$ at all other blocks. This recursively realizes the wavelet-modulus
cascade of Definition~\ref{def:scattering-transform}. Since
$j_1<\cdots<j_q$, only frequency-decreasing paths are
computed~\citep{bruna_invariant_2013}. Cascading the
blocks therefore computes the scattering tree using small spatial kernels and
progressively subsampled feature maps.

\emph{Convolutional realization and initialization.}
Suppose a scattering layer receives $C_{\mathrm{in}}$ input channels. In the classical
scattering transform, the wavelet filter bank is applied independently to each
channel. We realize this operation with a convolutional layer whose weights at
initialization satisfy
$$
K_{(c,k),c}=g_k,
\qquad
K_{(c,k),c'}=0
\quad\text{for }c'\neq c,
$$
where $(c,k)$ indexes the output obtained by applying filter $g_k$ to input
channel $c$. The resulting convolutional operator is block diagonal across
channels and therefore reproduces independent wavelet filtering exactly.
Together with the prescribed modulus, low-pass, and subsampling operations,
this defines the scattering realization $\Theta_0$.

\emph{Learnable channel mixing.}
Classical scattering does not mix information across different feature
channels. Motivated by previous learnable scattering architectures
\citep{zarka_separation_2021,DeepRotoScattering,guthphase}, we enlarge the
parameterization so that each output may receive contributions from a
prescribed neighborhood of input channels. Let $\mathcal N(c)$ denote the
channels allowed to contribute to descendants of channel $c$. The trainable
operation takes the form
$$
z_{c,k}(u)
=
\sum_{c'\in\mathcal N(c)}
(x_{c'}* K_{(c,k),c'})(2u).
$$

At initialization, the off-diagonal kernels remain zero,
$$
K_{(c,k),c}=g_k,
\qquad
K_{(c,k),c'}=0
\quad\text{for }c'\neq c,
$$
so the enlarged network still implements the predefined scattering transform
exactly. During fine-tuning, both the wavelet kernels and the permitted
off-diagonal kernels may change, allowing the representation to adapt to the
classification task. We restrict the mixing neighborhood rather than using
fully dense channel interactions in order to limit the number of trainable
parameters and preserve the local organization inherited from the scattering
tree.

Figure~\ref{fig:architecture} illustrates the resulting architecture. Thus,
$\Theta_0$ corresponds exactly to the predefined scattering computation,
whereas fine-tuning moves within a larger convolutional parameterization
initialized at this algorithmic realization.

\begin{figure}[t]
    \centering
    \includegraphics[width=0.95\linewidth]{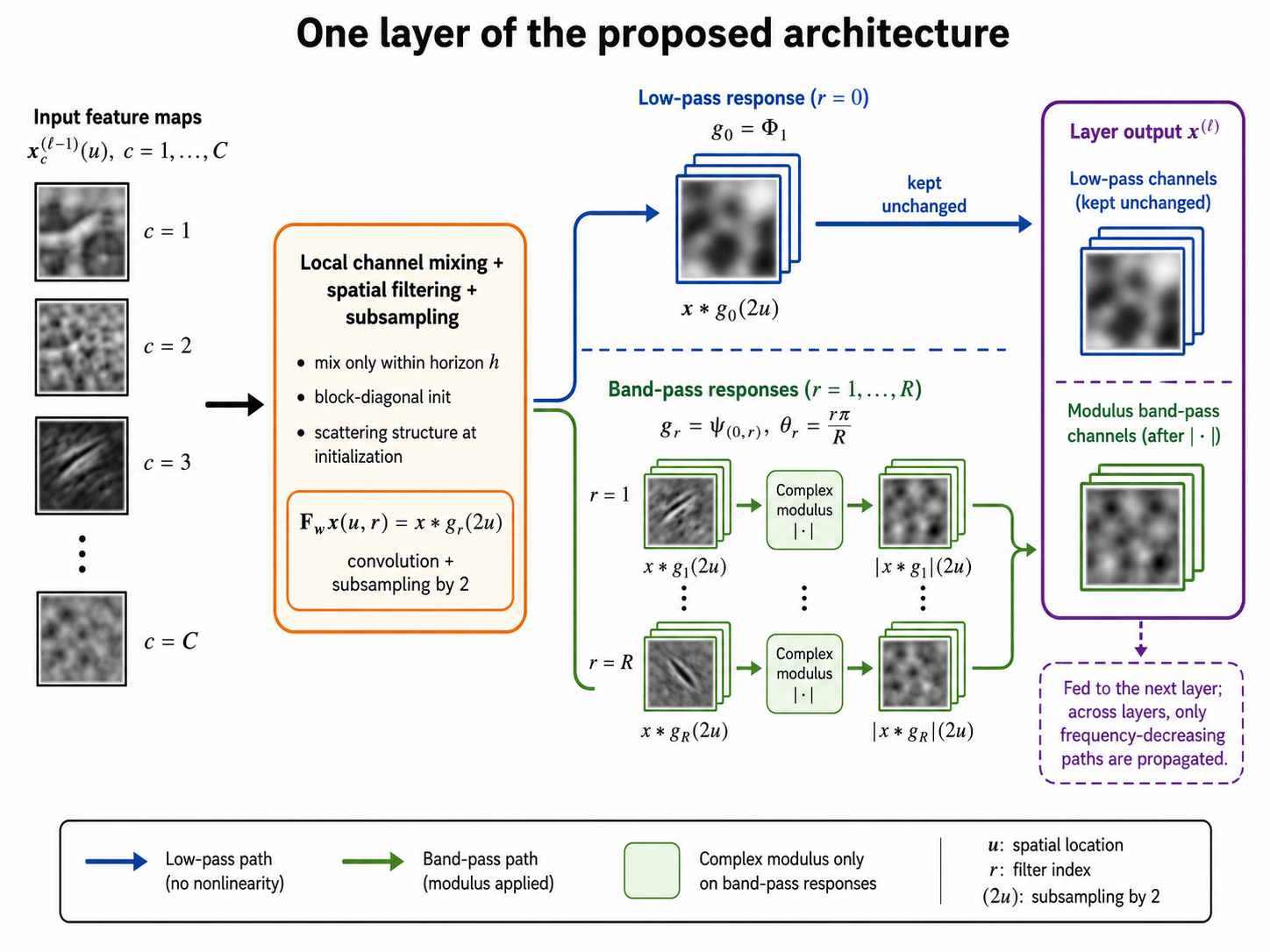}
    \caption{
    Trainable scattering architecture. At initialization, the convolutional
    operators reproduce independent wavelet filtering and hence the predefined
    scattering transform. Fine-tuning allows the wavelet filters and permitted
    cross-channel interactions to adapt from data.
    }
    \label{fig:architecture}
\end{figure}

\paragraph{Fixed and random-initialization baselines.}
The fixed scattering baseline is the predefined realization $\Theta_0$ used
without fine-tuning. The random-initialization baseline uses the same trainable
architecture as our model, but initializes its convolutional parameters
randomly instead of from the wavelet filters.
\paragraph{ScatNet B.A.}
We also compare with the learnable scattering architecture of
\citet{zarka_separation_2021}, denoted ScatNet B.A. Its representation is
computed as
$$
S_C
=
\prod_{j=1}^{J}
\left(F_j^T\rho_tF_j\right)
\left(P_j\,\mathrm{BN}\,\rho_rF_w\right).
$$
Here, $F_w$ denotes the fixed spatial wavelet transform, $\rho_r$ the
scattering nonlinearity, and $\mathrm{BN}$ channel-wise normalization.
The learned $1\times1$ projection $P_j$ reduces the scattering channels,
while $F_j$ is a learned $1\times1$ convolutional tight frame. The latter
expands the channel representation, applies the soft-thresholding
$\rho_t$, and projects it back through $F_j^T$. Thus, the spatial wavelet
filters remain fixed, while learning acts on transformations across the
scattering channels. Similar to our method, ScatNet B.A. includes batch normalization layers between scattering blocks, which improves training stability.
The interpretation of these learned operators in terms
of separation and concentration is developed in
\citet{zarka_separation_2021}.
\paragraph{Classification head.}

For all models, the final representation is passed through batch normalization
and then to a single-layer linear classifier.
When needed, global average
pooling is applied before the classifier to reduce the spatial dimensions.
Under the hypothesis that the method is a feature extractor which successfully disentangles representations between classes, a simple, linear classifier is sufficient to classify examples from their features. And the batch normalization layers are included for optimization stability and
consistency with standard training practice.

\paragraph{Parametric Scattering Networks.} \label{app:parametric_scattering}
\citet{gauthier_parametric_2022} propose another approach to learnable
scattering in which each wavelet is constrained to remain a Morlet wavelet and
learning is restricted to its scale, orientation, frequency, and aspect-ratio
parameters, rather than directly optimizing the values of the convolutional
kernels. Thus, the scattering filters are adapted through a low-dimensional
parameterization with only a small number of trainable degrees of freedom. In contrast, our model treats the spatial convolutional kernels as
trainable parameters and additionally permits cross-channel interactions,
resulting in a substantially more flexible, CNN-like parameterization.
Because our experiments focus on fine-tuning an overparameterized
convolutional network around the exact scattering realization, rather than a
low-dimensional wavelet parameterization, we do not include the parametric
Morlet model as a primary baseline.

\subsubsection{Training, Evaluation, and Reproducibility}
\label{app:training}
We review our training procedure and provide information to reproduce our numerical results, including hyperparameters and random seeds.
The implementation is publicly available on GitHub.\footnote{
\url{https://github.com/yangl1n/specific_algo_inter_on_texture}
}

\paragraph{Training Procedure}

The training procedure differs slightly between baseline methods.
Let \( \theta_{\psi} \) denote the parameters contained in the filters of the CNN (either the scattering filters or randomly initialized ones), and \( \theta_c\) denote the parameters contained in the classifier.
Let \( h_{\theta} \) denote the model with such parameters and \( \{ x_n, y_n \}_{n=1}^N\) be the set of training data.
For the randomly initialized model, the model was trained using standard, unconstrained gradient descent.
The minimization problem was \( \min_{\theta_{\psi}, \theta_c} \sum_{n} L(h_{\theta} (x_n), y_n) \).
For the initial scattering model, the filters were frozen and only the classifier head was trained.
The minimization problem was \( \min_{\theta_c} \sum_{n} L(h_{\theta} (x_n), y_n) \).
For the fine-tuned scattering model, the training proceeds in two phases. The first phase is identical to the training for the initial scattering model.
The second phase allows end-to-end training but introduces a regularization term to penalize distance that the parameters move from their initial values.
The minimization problem in the second phase is \( \min_{\theta_{\psi}, \theta_c} \sum_{n} L(h_{\theta} (x_n), y_n) + \lambda \| \theta_{\psi} - \theta_{\psi, 0} \|_2 \) where \( \theta_{\psi, 0} \) denotes the initial parameters of the scattering transform and \( \lambda \) controls the strength of the fine-tuning regularization.
For the ScatNet B.A., the training procedure from \citet{zarka_separation_2021} was used.
This is analogous to the training of the initial scattering transform, and the minimization problem is \( \min_{\theta_c} \sum_{n} L(h_{\theta} (x_n), y_n) \).

\paragraph{Architecture and Hyperparameters}

The model architecture used in experiments differs between datasets.
Architecture parameters including depth of network, number of wavelet orientations, the number of channels that were mixed by each wavelet filter, kernel size, per-block batch normalization, and whether a low-pass filter and global averaging are applied  before the classifier, were chosen by a sweep over hyperparameters on each dataset.
The learning rate was also tuned; separate learning rates are used for the classifier and the feature extractor.
All four reported models use a kernel size of 5 pixels, 4 wavelet orientations, a complex modulus as the linearity (as opposed to ReLU), and include a low-pass filter before passing features to the classifier.
In the sweep, larger kernels,  more wavelet orientations, and removing the final low-pass filter each increased accuracy on their own on some datasets, but by less than adding depth did, or by an amount that could not be separated from seed noise, while increasing the parameter count drastically, so none of them were adopted.
Depth (4 to 6 blocks), mixing width, global averaging, and per-block batch normalization vary by dataset.
The feature-extractor learning rate is \( 1.0 \times 10^{-4} \) on every dataset.
The classifier learning rate \( 0.001 \) on DTD, MinC-2500, and EuroSAT; on SuroTex, a learning-rate sweep selected \( 0.01 \).
The classifier's learning rate is reduced by half during fine-tuning.
For all datasets, a batch normalization layer was included between the scattering transform and the classifier, and a single fully-connected linear layer was used as the classifier head.
The randomly initialized network shares an architecture with the proposed method's architecture, so they are reported as a single row.
It is trained jointly in a single phase, with one learning rate for all parameters.

For the baseline method ScatNet B.A., none of the datasets reported here were used in  \citet{zarka_separation_2021}; we took the architecture used for ImageNet as a starting point and performed a sweep over architecture and hyperparameters. The following parameters were found to perform the best on all datasets: using 8 wavelet orientations, a kernel size of 1 for projection layers and frame layers, and applying a batch norm layer between scattering operators.
We also swept the depth of the network (number of scattering blocks), type of non-linearity, frame widths, projection widths, and learning rate.

We report the full architecture details for each dataset in \cref{tab:arch_config_deep_scatter,tab:arch_config_sep_conc_deepnets_a,tab:arch_config_sep_conc_deepnets_b}, excluding parameters that are the same across all datasets.

\begin{table}[t]
\centering
\caption[Architecture configuration for proposed method.]{Architecture configuration for proposed method: includes hyper-parameters that differ across datasets.
Rows sharing a configuration are merged (their architectures listed together). 
Values that are not listed were help fixed across all datasets (for example, angles (L) = 4, and low-pass last was always used).}
\label{tab:arch_config_deep_scatter}
\small
\begin{tabular}{lccccccc}
\toprule
Dataset & Params & Blocks & GAP & Mix. hor. & BN & LR step & WD \\
\midrule
DTD & 2.74M & 6 & True & 3 & False & 50 & 0.005 \\
EuroSat & 2.08M & 4 & False & 27 & False & 20 & 0.0001 \\
MinC-2500 & 2.56M & 6 & True & 3 & True & 50 & 0.005 \\
SuroTex & 112K & 4 & True & 3 & True & 20 & 0.005\\
\bottomrule
\end{tabular}
\end{table}

\begin{table}[t]
\centering
\caption[Architecture configuration for ScatNet B.A., part 1 of 2.]{Architecture configuration for ScatNet B.A., part 1 of 2: hyper-parameters that differ across the reported models. Rows sharing a configuration are merged (their architectures listed together).
Values that are not listed were held fixed across all datasets (for example, batch norm was always used).}
\label{tab:arch_config_sep_conc_deepnets_a}
\small
\begin{tabular}{lcccc}
\toprule
Dataset & Params & Blocks & Nonlin. & P proj. \\
\midrule
DTD & 26.8M & 6 & relu & 64/128/256/512/512/512 \\
EuroSat & 24.1M & 5 & softshrink & 64/128/256/512/512 \\
MinC-2500 & 9.47M & 4 & relu & 64/128/256/512 \\
SuroTex & 37.2M & 6 & softsrhink & 64/128/256/512/512/512 \\
\bottomrule
\end{tabular}
\end{table}

\begin{table}[t]
\centering
\caption{Architecture configuration for ScatNet B.A., part 2 of 2.}
\label{tab:arch_config_sep_conc_deepnets_b}
\small
\begin{tabular}{lccccc}
\toprule
Dataset & Frame w. & LR &  LR step & Epochs & Batch \\
\midrule
DTD & 512/1024/1024/2048/2048/2048 & 0.01 & 60 & 200 & 128 \\
EuroSat & 1024/2048/4096/8192/8192 & 0.01 & 70 & 300 & 128 \\
MinC-2500 & 512/1024/1024/2048 & 0.01 & 60 & 200 & 32 \\
SuroTex & 1024/2048/4096/8192/8192/8192 & 0.001 & 70 & 300  & 25 \\
\bottomrule
\end{tabular}
\end{table}

\paragraph{Statistical significance and reproducibility}
Each configuration was independently run across three different random seeds (42, 43, 44), to capture variability in initialization and dataset sampling.
Each figure of results displays the statistical significance.
\cref{fig:acc-vs-train-size,fig:admissibility} displays the mean test accuracy across seeds with a shaded band of \( \pm 1 \) standard deviation.
\cref{fig:filter-dist-vs-gen-gap,fig:test-acc-vs-gen-gap,fig:gen-gap-weight-decay,fig:gen-gap-anchor,fig:rho-vs-filter-distance} plot seed means only.
\cref{fig:skew-neural-collapse} shows the seed mean with \( \pm 1 \) standard deviation error bars on both axes.
Tabular results with standard deviation are reported in \Cref{app:more-results}.

\paragraph{Use of existing assets} We use existing assets of data and code in our paper.
DTD is shared without a license; we used the current version DTD-R1.0.1, and is available open-source from the creators \citep{cimpoi14describing}.
EuroSat is shared under an MIT license \citep{helber2019eurosat}.
MinC-2500 and SuroTex are shared under CC BY 4.0 licenses \citep{bell2015material,putra2025surotex}. The texture-generation code builds on the implementation released with
\citet{brochard_generalized_2021}, which is under an Apache-2.0 license.
We also use existing code from \citet{zarka_separation_2021} to implement the ScatNet B.A. baseline, with minor updates to allow training on a subset of data and report additional statistics, such as within-class variance.
The original code is available at \url{https://github.com/j-zarka/separation_concentration_deepnets/} and is shared under a BSD 3-clause license that allows modification and private and commercial use.

\subsection{Texture Classification Results}\label{app:more-results}

\subsubsection{Test Accuracy versus Generalization Gap}

\begin{figure}
    \centering
    \includegraphics[width=1.0\linewidth]{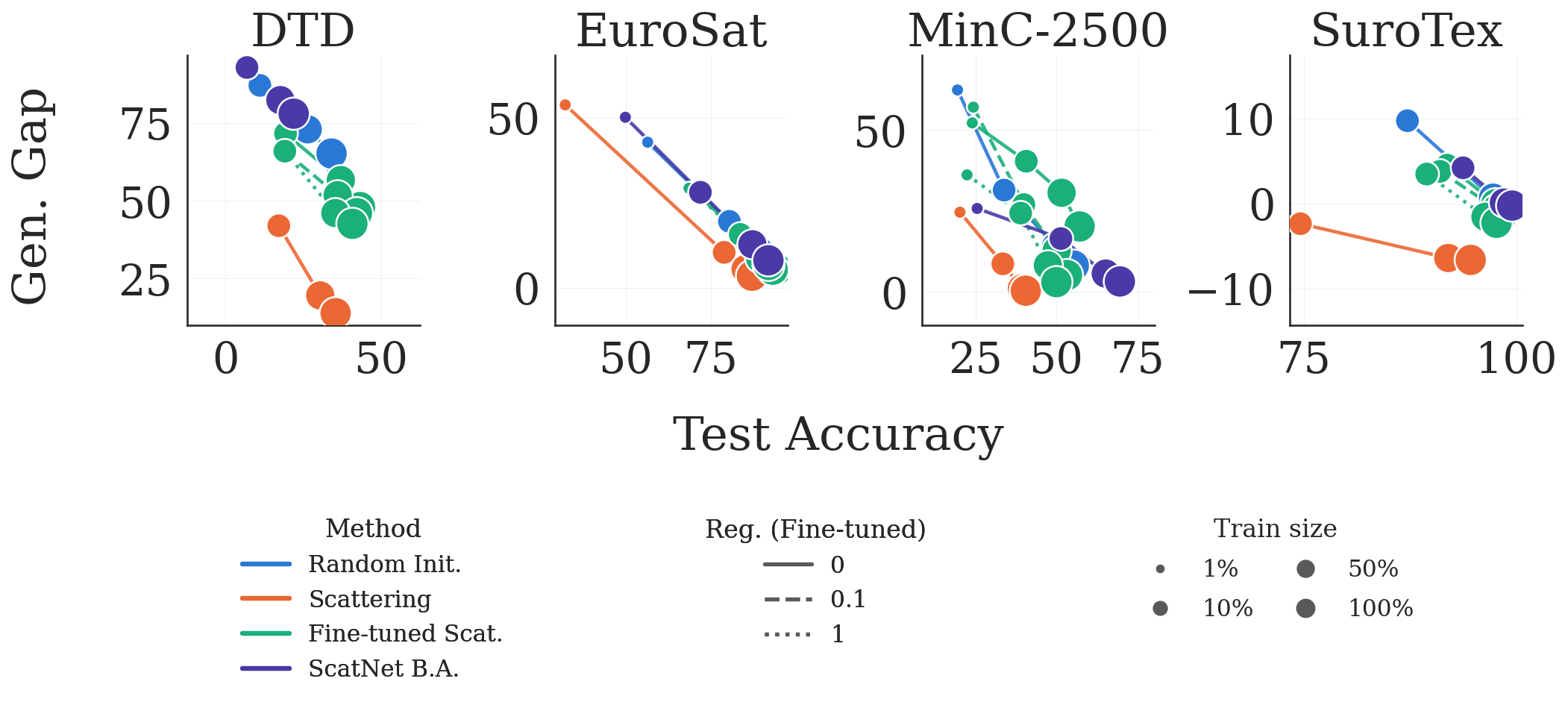}
    \caption{Test accuracy versus generalization gap. Color denotes method, line style denotes regularization strength, and marker size denotes percentage of training data used. Increasing regularization strength decreases generalization error on DTD, MinC-2500, and SuroTex. }
    \label{fig:test-acc-vs-gen-gap}
\end{figure}

\Cref{fig:test-acc-vs-gen-gap} displays test accuracy versus generalization gap.
Generalization gap is calculated as test accuracy minus training accuracy.
The ideal model lies in the bottom right of each plot, which would indicate high accuracy and a low generalization gap.
Across datasets and methods, a lower generalization gap is correlated with higher test accuracy.
The original scattering transform has the smallest generalization gap across all four datasets. 

This matches standard generalization bounds, which depend on the complexity of the hypothesis class.
The scattering network without fine-tuning only trains the linear classifier, so it has a small fraction of the number of parameters of the other methods and its associated hypothesis class has much lower complexity.
Increasing regularization strength, corresponding to decreased fine-tuning distance, also decreases the generalization gap on DTD, MinC-2500, and SuroTex.
This observation also aligns with \cref{coro:general_ft_bound}, which predicts a smaller fine-tuning distance will correlate with lower generalization gap.
Regularization does not have an effect on the fine-tuned scattering network on EuroSat, where all of the methods achieve near perfect accuracy on the training dataset.
This suggests that the initial wavelet parameters already perform well and do not need to move during training, decreasing the effect of anchoring the parameters near their initial value.
Finally, the fine-tuned scattering network improves on the accuracy of the original scattering transform. 
The fine-tuned scattering network also outperforms the randomly initialized model in accuracy and generalization gap, with a stronger advantage when training data is limited.
This also supports \cref{coro:general_ft_bound}; the generalization gap can be controlled either by reducing the fine-tuning distance or by increasing the sample size, so the fine-tuned model has a larger advantage when training data is limited.

\subsubsection{Regularization Ablation}\label{app:regularization}

\begin{figure}
    \centering
    \includegraphics[width=1.0\linewidth]{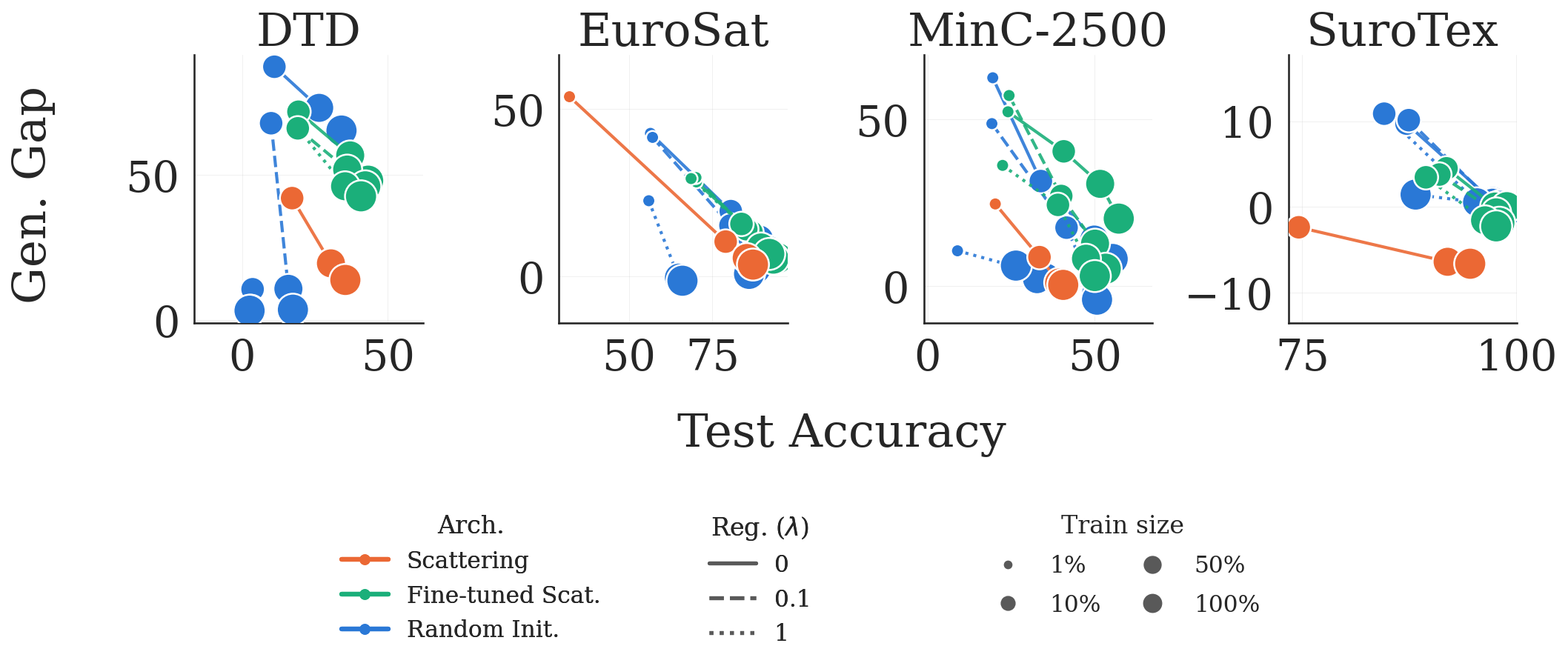}
    \caption{
    Test accuracy versus generalization gap, with weight decay regularization on randomly initialized model.
    Color denotes method, line weight denotes regularization strength, and size denotes percentage of training data. Weight decay controls the generalization gap at the cost of reducing test accuracy on DTD, EuroSat, and MinC-2500.}
    \label{fig:gen-gap-weight-decay}
\end{figure}

It is well-known that regularizing a model reduces the generalization gap.
\Cref{fig:gen-gap-weight-decay} ablates the effect of controlling fine-tuning distance versus standard weight decay regularization on the randomly initialized network.
The two differ in the parameters they are anchored to during training: fine-tuning distance penalizes movement away from the initial parameters, such as the scattering prior, while weight decay penalizes distance from the origin.

Weight decay reduces the generalization gap for the randomly initialized model, but it does so by reducing the test accuracy, in some cases to near 0, as on DTD.
The model achieves low generalization error by performing equally poorly on the training and testing datasets.
In contrast, regularizing the fine-tuning distance on the fine-tuned scattering network produces a smaller reduction in generalization gap but at far less cost to test accuracy.
The original scattering transform, without any fine-tuning, can be viewed as applying fine-tuning regularization with \( \lambda = \infty\), and it does not demonstrate the same collapse in test accuracy as the randomly initialized model does with \( \lambda = 1\).
\Cref{fig:gen-gap-weight-decay} indicates that controlling the fine-tuning distance can control generalization error without destroying accuracy; anchoring the parameters to a meaningful initial distribution that disentangles the data does not have a negative effect on test accuracy.
On the other hand, anchoring the parameters to zero, as weight decay regularization does, negatively impacts test accuracy.

\begin{figure}
    \centering
    \includegraphics[width=1.0\linewidth]{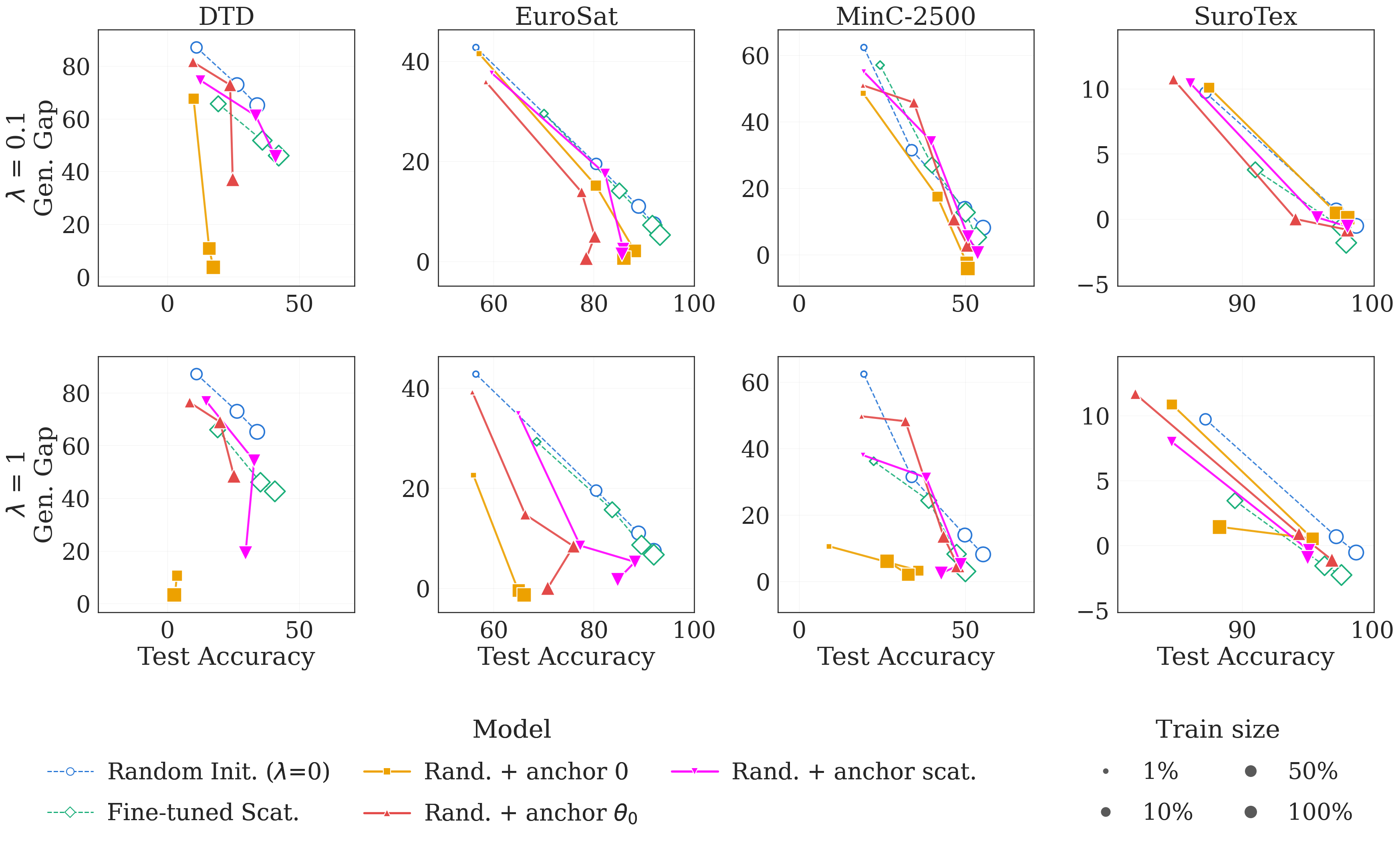}
    \caption{ 
    Test accuracy versus generalization gap, comparing effect of different parameter anchors on randomly initialized model. Within each subplot, color indicates method and parameter anchor, and marker size indicates percentage of training data used. The randomly initialized model with scattering anchor has higher generalization error than the fine-tuned scattering model, likely because the fine-tuning distance is greater on DTD and MinC-2500. 
    }
    \label{fig:gen-gap-anchor}
\end{figure}

\Cref{fig:gen-gap-anchor} illustrates the effect of different anchors on the randomly initialized model and separates two overlapping effects in the fine-tuned model; the benefit conferred by a meaningful prior and the benefit conferred by controlling fine-tuning distance.
Each subplot shows the randomly initialized model with regularization penalizing the distance from three parameter anchors, the origin, the random initial parameters, and the initial scattering parameters, alongside the randomly initialized model without regularization and the fine-tuned scattering model.

Anchoring the randomly initialized model to the scattering prior improves the test accuracy and avoids the steep decline in test accuracy observed with standard weight decay, which demonstrates the first effect, that anchoring to a meaningful prior improves accuracy, in isolation.
However, the fine-tuned scattering transform demonstrates a lower generalization gap on DTD and SuroTex.
It anchors to its own initial parameters, while the randomly initialized model with a scattering anchor starts from one set of parameters and pulls towards another during training.
The fine-tuning distance is therefor much smaller for the fine-tuned scattering network, giving it the second effect, control of generalization gap via control of fine-tuning distance, as well.
Although the first effect is on test accuracy and the second effect is on the generalization gap, they can amplify each other, since test accuracy and generalization gap are highly correlated, as shown in \cref{fig:test-acc-vs-gen-gap}.
The fine-tuned scattering network benefits from a meaningful prior and from restricted fine-tuning distance, while the randomly initialized model with a scattering anchor benefits from the first effect alone.

\subsubsection{Neural Collapse}\label{app:neural-collapse}

\begin{figure}
    \centering
    \includegraphics[width=1.0\linewidth]{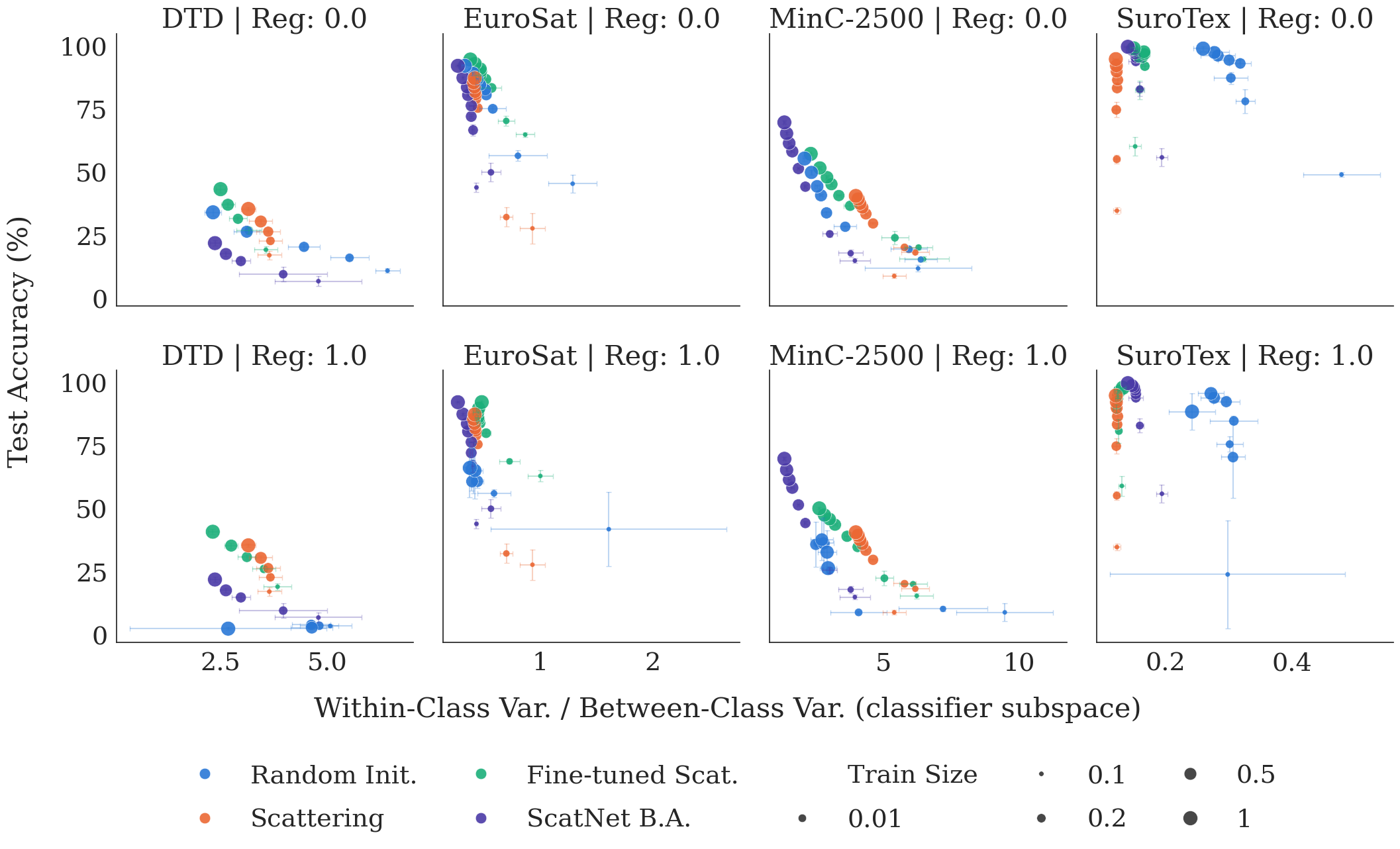}
    \caption{ 
    Normalized within-class variance in classifier subspace versus test accuracy. Within sub-plots, color denotes method, marker size denotes percentage of training data used, and cross-hatches denote standard deviation. Rows denote regularization strength and columns denote datasets.}
    \label{fig:skew-neural-collapse}
\end{figure}

\Cref{fig:skew-neural-collapse} displays normalized within-class variance against test accuracy.
We report a ratio of traces, \( \text{tr} (\Sigma_w M) / \text{tr}(\Sigma_b M)\), following the convention in \citet{tirer2023perturbation, beaglehole2024average, jacot2025wide}.
We will refer to this quantity as normalized within-class variance.
\( \Sigma_w \) denotes the within-class covariance of the features, averaged over classes, and \( \Sigma_b \) denotes the covariance of the class means about their center:
\[
\Sigma_w = \frac{1}{K} \sum_{c=1}^{K} \frac{1}{n_c} \sum_{i : y_i = c}
      (h_i - \mu_c)(h_i - \mu_c)^{\top},
  \qquad
  \Sigma_b = \frac{1}{K} \sum_{c=1}^{K} (\mu_c - \bar{\mu})(\mu_c - \bar{\mu})^{\top},
\]
where \( h_i \) is the feature of sample \( i\), \( \mu_c\) is the mean of class \( c\), \( \bar{\mu}\) is the mean of the class means, and \( K \) is the number of classes.
Standard within-class variance sets \( M \) equal to the identity matrix; however for settings with many more feature dimensions than classes, this can penalize a model for variance in directions that are invisible to the classifier, which penalizes a feature representation that is well-separated with respect to the geometry of the classifier.
Therefore, in \cref{fig:skew-neural-collapse}, we set \( M = V V^T \), where the columns of \( V \) are an orthonormal basis for the row space of the classifier weights after their mean row is removed, taken as the Jacobian of the head with respect to \( h_i \) so that its normalization scale and any pooling are included. 
Both traces are then taken inside the classifier subspace, on the projected features \( V^T h_i \).
Feature variation orthogonal to that subspace does not change the model's predictions, so we do not penalize the model for variance in these directions.
The ideal model resides in the top left corner of each subplot, with low within-class variance, which indicates a high degree of neural collapse, and high test accuracy.
To understand the effect of controlling fine-tuning distance on neural collapse, \cref{fig:skew-neural-collapse} displays results without regularization in the top row and with regularization (\(\lambda = 1\)) in the bottom row.
The results for the scattering network without fine-tuning and for ScatNet B.A. are repeated between rows because they do not have regularization.
Results for the fine-tuned scattering network regularize distance from the initial parameters and results for the randomly initialized network apply weight decay regularization.

In \cref{fig:skew-neural-collapse}, normalized within-class variance and test accuracy are correlated with each other.
The fine-tuned scattering network outperforms the original scattering network on both metrics on DTD and MinC-2500, although ScatNet B.A. outperforms the fine-tuned scattering network on MinC-2500.
Unsurprisingly, stronger regularization  anchors the normalized of the fine-tuned scattering network closer to the normalized of the original scattering transform, especially on SuroTex and EuroSat.
SuroTex is the only dataset where the original scattering transform has significantly lower normalized than all other methods, which aligns with the fact that SuroTex is closest to the classical texture datasets that the scattering transform describes well before fine-tuning.
\cref{fig:skew-neural-collapse} suggests that regularizing fine-tuning distance helps preserve desirable properties of the scattering transform, such as neural collapse, in addition to controlling the generalization gap.

\subsubsection{Model Interpretability Results}\label{subsec:interpret-results}

A major advantage of the proposed method is that not only is the initial model interpretable, but the final model should also be interpretable.
The scattering network's interpretability follows from the fact that its filters are defined by a wavelet filter bank.
In order to verify that the fine-tuned scattering network maintains a reasonable interpretation, we report several metrics that measure how close the trained filters of a model are to wavelets in \cref{fig:admissibility,fig:rho-vs-filter-distance}, and how much the parameters of the overall filter bank change in \cref{fig:interp-params-shift}.

\begin{figure}
    \centering
    \includegraphics[width=1.0\linewidth]{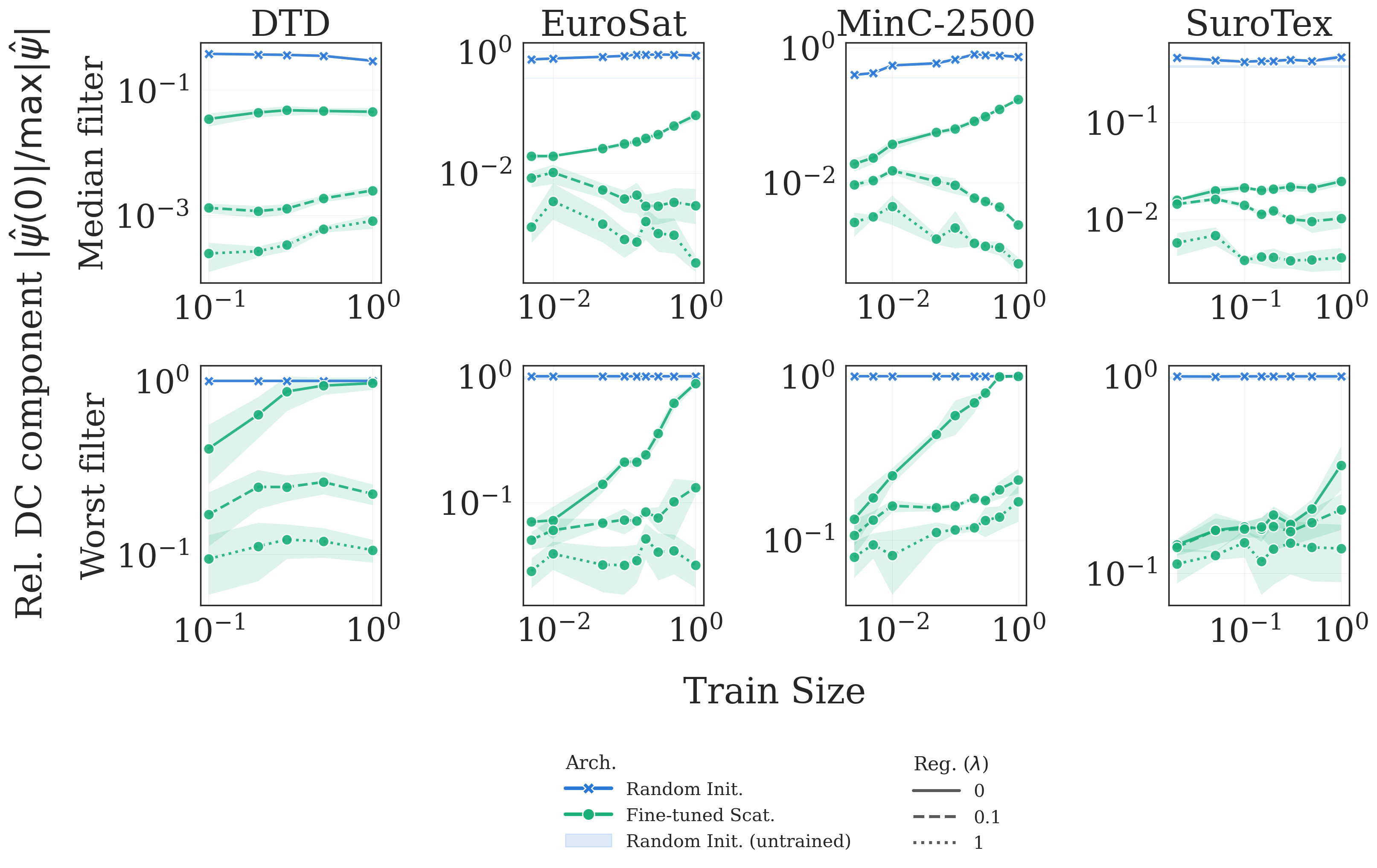}
    \caption{Percentage of data used during training versus relative magnitude of zero-frequency component. Wavelets must satisfy an admissibility condition, that the zero-frequency component has magnitude zero. Color denotes method, line style denotes regularization strength. Increasing regularization strength decreases the relative magnitude of the zero-frequency component on all datasets.}
    \label{fig:admissibility}
\end{figure}

\Cref{fig:admissibility} displays the relative magnitude of the zero-frequency
component against the percentage of data used in training, for the median
filter (top row) and the worst filter (bottom row) of each model.
Wavelets must satisfy an admissibility condition: a wavelet \( \psi \) has zero
mean, \( \int \psi(x)\,dx = 0 \), or equivalently a vanishing zero-frequency
component.
Because the condition is common to all wavelet families, this metric makes no
assumption about which wavelet the filters resemble, but it is correspondingly
weak: it can rule a filter out as a wavelet, not certify it as one.

The randomly initialized model fails the condition outright, and its median and
worst-case filters have similar values, indicating that large zero-frequency
components are the norm rather than the exception among its filters.
Its learned filters therefore cannot be interpreted as wavelets.
The fine-tuned scattering model retains a much smaller zero-frequency component at the median, and stronger regularization shrinks it further.
Its worst-case filter degrades as the training set grows, but penalizing the fine-tuning distance keeps that degradation bounded.
Fine-tuning thus leaves the network interpretable as a family of wavelet-based filters rather than destroying that structure.

\begin{figure}
    \centering
    \includegraphics[width=1.0\linewidth]{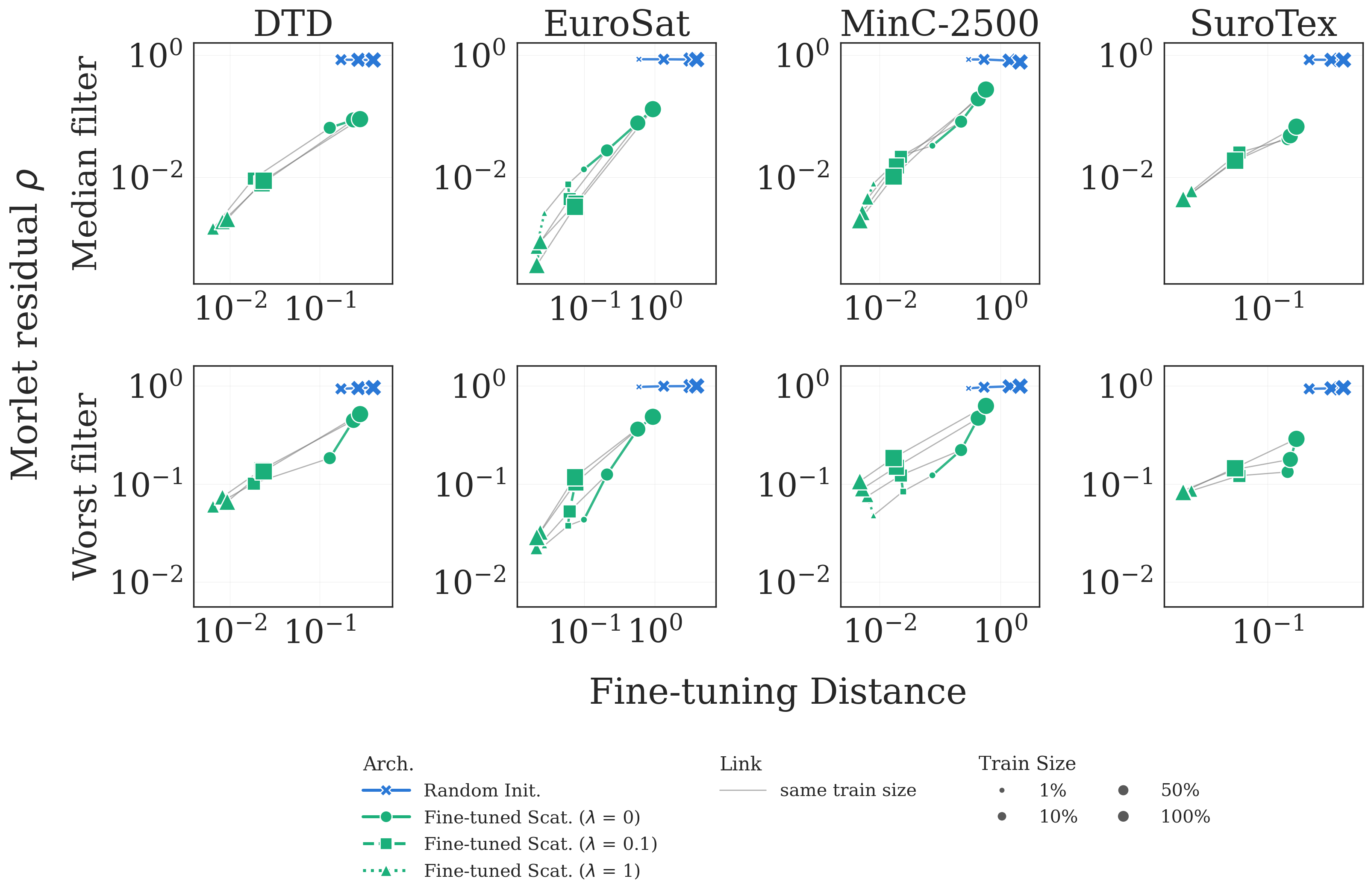}
    \caption{Fine-tuning distance versus \( \rho\) (similarity to best-fit Morlet filter). Color denotes method and marker style denotes regularization strength. Marker size denotes the percentage of data used during training.
    Smaller fine-tuning distance is correlated with more similarity to a Morlet filter across all four datasets.}
    \label{fig:rho-vs-filter-distance}
\end{figure}

Having ruled out the randomly initialized filters as wavelets, we now ask how wavelet-like the fine-tuned filters are.
Since the scattering network is initialized with Morlet filters, we measure each filter's deviation from the closest member of the Morlet family.
For a given fine-tuned filter, we fit the orientation, central frequency, envelope width, and aspect ratio of the Morlet filter that most closely matches it, and define \( \rho \) as the \( L^2 \) distance between the two, normalized by the norm of the fine-tuned filter.
Unlike the admissibility condition, this measure is specific to the family the
network was initialized from, and it is positive evidence rather than a
necessary condition: a small \( \rho \) demonstrates that the filter is still a Morlet wavelet, accounting for the possibility that its parameters are different that the parameters that generated the original Morlet filter.

\Cref{fig:rho-vs-filter-distance} displays \( \rho \) against the fine-tuning
distance.
Filters that move less remain closer to the Morlet family, as expected.
The informative comparison is vertical: between filters that traveled the same distance from their initialization but differ in how far they ended up from the
family.
At matched distance, the fine-tuned filters remain substantially closer to the Morlet family than the randomly initialized ones, so their similarity to wavelets is not merely a consequence of having moved less.
Fine-tuning not only keeps the filters near their starting point, it leaves them within the family of filters from which they were initialized.

\begin{figure}
    \centering
    \includegraphics[width=1.0\linewidth]{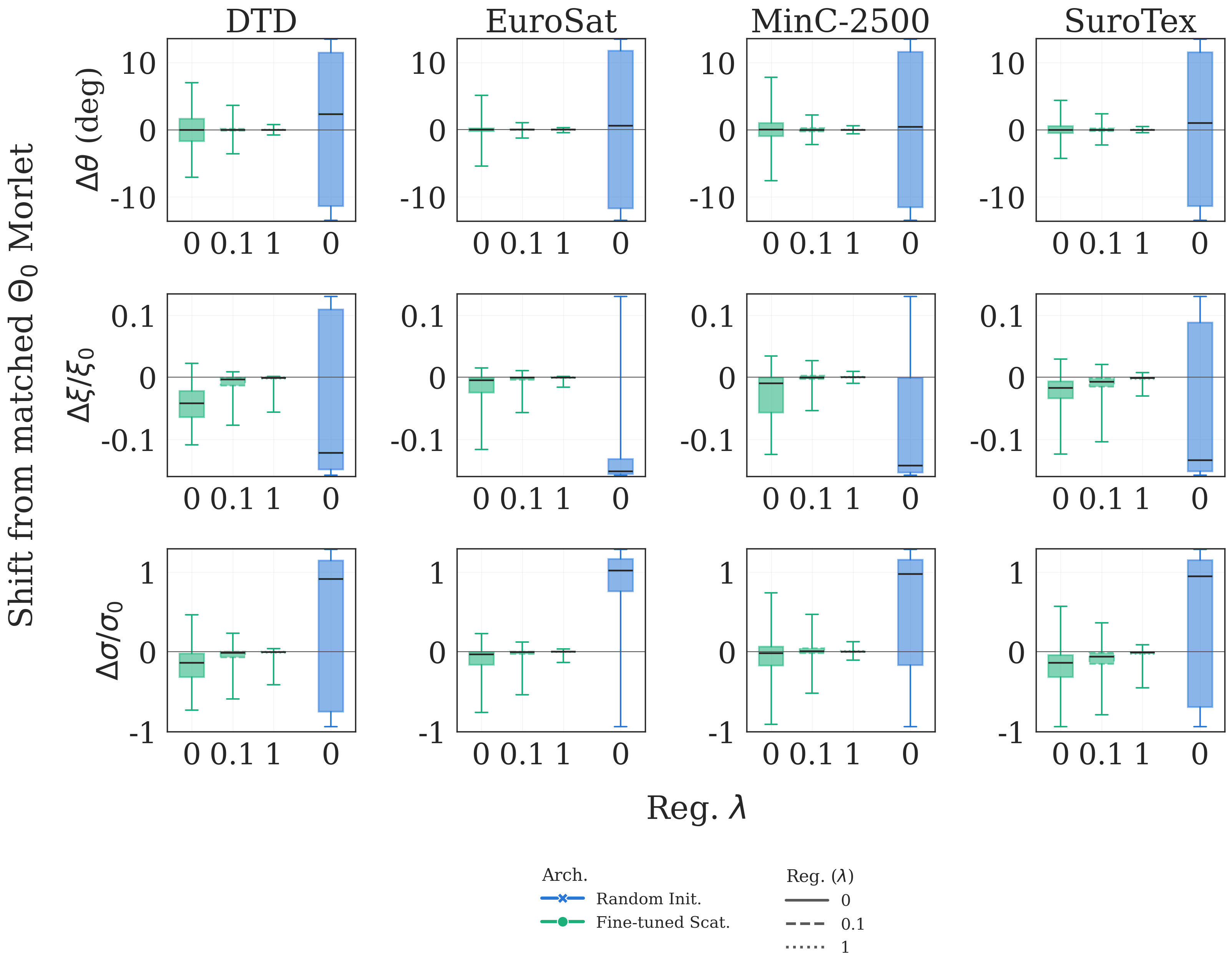}
    \caption{Relative change in the fitted Morlet parameters during fine-tuning. For each fine-tuned filter, we compare the parameters of its best-fit Morlet wavelet
    to those of the matched filter in the initial bank. Regularizing the fine-tuning distance keeps the parameters closer to their initial values across datasets.}
    \label{fig:interp-params-shift}
\end{figure}

Verifying that each individual filter remains wavelet-like after fine-tuning is necessary but not sufficient to ensure that the overall model remains interpretable; the scattering transform is not an arbitrary collection of wavelets, but one whose elements are distributed across scales and orientations.
Therefore, in \cref{fig:interp-params-shift}, we report how far the parameters of the fine-tuned filter bank have shifted during training.
For each parameter (orientation \( \theta\), central frequency \(\xi\), and bandwidth \( \sigma\)) the shift is the difference between the best-fit Morlet wavelet used to compute \( \rho \) and the matched filter in the initial bank.
Matching filters rather than comparing them by index ensures that models whose filters are permuted during training are not penalized for the permutation.
It also means the shifts are of limited use for the randomly initialized model, whose large \( \rho \) values indicate that the best-fit Morlet wavelet is a poor approximation of the true filter.

\cref{fig:interp-params-shift} demonstrates that fine-tuning successfully controls the change in fitted parameters.
With \( \lambda = 1 \), the relative shifts are close to zero across datasets and parameters, whereas with \( \lambda = 0 \) they are substantially larger on DTD, Minc-2500, and SuroTex, although the parameter change is small at all regularization levels for EuroSat.
Together with the per-filter results above, this is consistent with fine-tuning preserving not only the character of the individual wavelets but the organization of the bank they form.

\subsubsection{Full Numerical Results}

To allow more precise reporting of our results, we include tables comparing test accuracy, generalization error, and normalized within-class variance in the classifier subspace.
We report the mean and standard deviation across 3 random seeds for all datasets except DTD, where we report the mean and 1 standard deviation across 3 random seeds and 5 partitions of the data.
The factors of variability include the split of data to include in the training subset, train/val split, and model initialization.
While the filters for the scattering transform and ScatNet B.A. are deterministic, the classifier is initialized randomly, as well as the analysis layers in ScatNet B.A.


{%
\setlength{\LTcapwidth}{\textwidth}%
\begin{longtable}{llccccc}
\caption{Test accuracy, generalization gap and normalized within-class variance in classifier subspace (NC1) across architectures for DTD, by training-set fraction. Each cell is the mean over seeds, with the standard deviation in parentheses, set smaller. Best test accuracy per column in bold.}
\label{tab:bestof_dtd}\\
\toprule
Model & Metric & 10\% & 20\% & 30\% & 50\% & 100\% \\
\midrule
\endfirsthead
\caption*{\textit{Table \thetable{} (continued).}}\\
\toprule
Model & Metric & 10\% & 20\% & 30\% & 50\% & 100\% \\
\midrule
\endhead
\midrule
\multicolumn{7}{r}{\footnotesize\textit{Continued on next page}}\\
\endfoot
\bottomrule
\endlastfoot
\multirow{3}{1.75cm}{\raggedright Random Init. (L2=0)} & Test acc. & 11.06 {\scriptsize (0.67)} & 16.13 {\scriptsize (0.67)} & 20.45 {\scriptsize (0.40)} & 26.43 {\scriptsize (0.31)} & 34.10 {\scriptsize (0.92)} \\*
 & Gen. Gap & 87.10 {\scriptsize (2.73)} & 83.26 {\scriptsize (1.04)} & 79.46 {\scriptsize (0.39)} & 72.98 {\scriptsize (1.00)} & 65.19 {\scriptsize (0.54)} \\*
 & NC1 & 6.42 {\scriptsize (0.21)} & 5.53 {\scriptsize (0.30)} & 4.46 {\scriptsize (0.29)} & 3.12 {\scriptsize (0.21)} & 2.33 {\scriptsize (0.17)} \\
\midrule
\multirow{3}{1.75cm}{\raggedright Random Init. (L2=0.1)} & Test acc. & 9.94 {\scriptsize (0.51)} & 13.44 {\scriptsize (1.15)} & 14.66 {\scriptsize (0.69)} & 15.87 {\scriptsize (1.62)} & 17.42 {\scriptsize (1.12)} \\*
 & Gen. Gap & 67.70 {\scriptsize (6.84)} & 49.16 {\scriptsize (10.15)} & 25.21 {\scriptsize (6.66)} & 10.82 {\scriptsize (4.18)} & 3.66 {\scriptsize (2.84)} \\*
 & NC1 & 5.55 {\scriptsize (0.30)} & 5.37 {\scriptsize (0.87)} & 4.78 {\scriptsize (0.43)} & 4.56 {\scriptsize (0.52)} & 4.33 {\scriptsize (0.33)} \\
\midrule
\multirow{3}{1.75cm}{\raggedright Random Init. (L2=1)} & Test acc. & 3.57 {\scriptsize (0.36)} & 3.68 {\scriptsize (0.45)} & 4.06 {\scriptsize (0.39)} & 2.95 {\scriptsize (0.32)} & 2.53 {\scriptsize (0.40)} \\*
 & Gen. Gap & 10.68 {\scriptsize (1.87)} & 5.28 {\scriptsize (3.00)} & 1.13 {\scriptsize (2.13)} & 3.67 {\scriptsize (2.53)} & 3.32 {\scriptsize (0.73)} \\*
 & NC1 & 5.08 {\scriptsize (0.35)} & 4.82 {\scriptsize (0.33)} & 4.63 {\scriptsize (0.42)} & 4.64 {\scriptsize (0.39)} & 2.68 {\scriptsize (1.92)} \\
\midrule
\multirow{3}{1.75cm}{\raggedright Scattering} & Test acc. & 17.15 {\scriptsize (1.14)} & 22.83 {\scriptsize (0.64)} & 26.46 {\scriptsize (0.41)} & 30.47 {\scriptsize (0.85)} & 35.42 {\scriptsize (0.35)} \\*
 & Gen. Gap & 41.98 {\scriptsize (1.63)} & 30.97 {\scriptsize (1.08)} & 25.20 {\scriptsize (0.88)} & 19.57 {\scriptsize (0.68)} & 13.88 {\scriptsize (0.46)} \\*
 & NC1 & 3.65 {\scriptsize (0.29)} & 3.67 {\scriptsize (0.28)} & 3.62 {\scriptsize (0.29)} & 3.45 {\scriptsize (0.28)} & 3.16 {\scriptsize (0.17)} \\
\midrule
\multirow{3}{1.75cm}{\raggedright Fine-tuned Scat. (L2=0)} & Test acc. & \textbf{19.32} {\scriptsize (0.50)} & \textbf{26.91} {\scriptsize (0.44)} & 31.54 {\scriptsize (0.35)} & \textbf{37.09} {\scriptsize (0.73)} & \textbf{43.25} {\scriptsize (1.05)} \\*
 & Gen. Gap & 71.65 {\scriptsize (6.23)} & 66.49 {\scriptsize (5.03)} & 62.23 {\scriptsize (3.56)} & 56.65 {\scriptsize (4.39)} & 47.96 {\scriptsize (7.03)} \\*
 & NC1 & 3.57 {\scriptsize (0.27)} & 3.17 {\scriptsize (0.30)} & 2.91 {\scriptsize (0.22)} & 2.68 {\scriptsize (0.16)} & 2.50 {\scriptsize (0.09)} \\
\midrule
\multirow{3}{1.75cm}{\raggedright Fine-tuned Scat. (L2=0.1)} & Test acc. & 19.31 {\scriptsize (0.31)} & 26.70 {\scriptsize (0.63)} & \textbf{31.60} {\scriptsize (0.18)} & 36.06 {\scriptsize (1.02)} & 42.26 {\scriptsize (0.96)} \\*
 & Gen. Gap & 65.70 {\scriptsize (5.56)} & 62.71 {\scriptsize (7.72)} & 60.72 {\scriptsize (2.46)} & 51.79 {\scriptsize (4.96)} & 46.02 {\scriptsize (3.09)} \\*
 & NC1 & 3.70 {\scriptsize (0.31)} & 3.28 {\scriptsize (0.30)} & 2.99 {\scriptsize (0.20)} & 2.66 {\scriptsize (0.16)} & 2.33 {\scriptsize (0.10)} \\
\midrule
\multirow{3}{1.75cm}{\raggedright Fine-tuned Scat. (L2=1)} & Test acc. & 19.08 {\scriptsize (0.51)} & 26.11 {\scriptsize (0.37)} & 30.82 {\scriptsize (0.58)} & 35.33 {\scriptsize (0.85)} & 40.80 {\scriptsize (0.94)} \\*
 & Gen. Gap & 65.95 {\scriptsize (4.26)} & 57.94 {\scriptsize (8.44)} & 58.09 {\scriptsize (5.46)} & 46.05 {\scriptsize (4.86)} & 42.60 {\scriptsize (3.34)} \\*
 & NC1 & 3.84 {\scriptsize (0.33)} & 3.52 {\scriptsize (0.24)} & 3.12 {\scriptsize (0.20)} & 2.76 {\scriptsize (0.12)} & 2.32 {\scriptsize (0.10)} \\
\midrule
\multirow{3}{1.75cm}{\raggedright ScatNet B.A.} & Test acc. & 6.91 {\scriptsize (0.77)} & 9.66 {\scriptsize (1.14)} & 14.85 {\scriptsize (0.76)} & 17.65 {\scriptsize (0.73)} & 21.93 {\scriptsize (0.61)} \\*
 & Gen. Gap & 92.85 {\scriptsize (0.60)} & 90.23 {\scriptsize (0.98)} & 85.15 {\scriptsize (0.76)} & 82.35 {\scriptsize (0.73)} & 78.00 {\scriptsize (0.68)} \\*
 & NC1 & 4.80 {\scriptsize (0.56)} & 3.97 {\scriptsize (0.60)} & 2.98 {\scriptsize (0.17)} & 2.63 {\scriptsize (0.06)} & 2.37 {\scriptsize (0.10)} \\
\end{longtable}
}%

{%
\setlength{\LTcapwidth}{\textwidth}%
\begin{longtable}{llccccc}
\caption{Test accuracy, generalization gap and normalized within-class variance in classifier subspace (NC1) across architectures for EuroSat (Part 1 of 2), by training-set fraction. Each cell is the mean over seeds, with the standard deviation in parentheses, set smaller. Best test accuracy per column in bold.}
\label{tab:bestof_eurosat_a}\\
\toprule
Model & Metric & 0.5\% & 1\% & 5\% & 10\% & 15\% \\
\midrule
\endfirsthead
\caption*{\textit{Table \thetable{} (part 1 of 2, continued).}}\\
\toprule
Model & Metric & 0.5\% & 1\% & 5\% & 10\% & 15\% \\
\midrule
\endhead
\midrule
\multicolumn{7}{r}{\footnotesize\textit{Continued on next page}}\\
\endfoot
\bottomrule
\endlastfoot
\multirow{3}{1.75cm}{\raggedright Random Init. (L2=0)} & Test acc. & 45.40 {\scriptsize (3.57)} & 56.45 {\scriptsize (2.16)} & 74.98 {\scriptsize (1.21)} & 80.48 {\scriptsize (0.27)} & 82.57 {\scriptsize (0.83)} \\*
 & Gen. Gap & 54.60 {\scriptsize (3.57)} & 42.77 {\scriptsize (2.50)} & 24.99 {\scriptsize (1.27)} & 19.51 {\scriptsize (0.28)} & 17.43 {\scriptsize (0.83)} \\*
 & NC1 & 1.29 {\scriptsize (0.22)} & 0.80 {\scriptsize (0.26)} & 0.58 {\scriptsize (0.12)} & 0.52 {\scriptsize (0.04)} & 0.51 {\scriptsize (0.05)} \\
\midrule
\multirow{3}{1.75cm}{\raggedright Random Init. (L2=0.1)} & Test acc. & 24.78 {\scriptsize (4.44)} & 57.07 {\scriptsize (3.67)} & 73.64 {\scriptsize (2.94)} & 80.38 {\scriptsize (0.71)} & 83.65 {\scriptsize (0.48)} \\*
 & Gen. Gap & 13.37 {\scriptsize (2.94)} & 41.57 {\scriptsize (2.78)} & 19.65 {\scriptsize (7.42)} & 15.22 {\scriptsize (1.82)} & 10.37 {\scriptsize (0.54)} \\*
 & NC1 & 2.56 {\scriptsize (1.10)} & 0.87 {\scriptsize (0.25)} & 0.49 {\scriptsize (0.04)} & 0.49 {\scriptsize (0.06)} & 0.44 {\scriptsize (0.03)} \\
\midrule
\multirow{3}{1.75cm}{\raggedright Random Init. (L2=1)} & Test acc. & 41.80 {\scriptsize (14.68)} & 55.94 {\scriptsize (1.68)} & 67.33 {\scriptsize (0.13)} & 64.91 {\scriptsize (9.05)} & 60.73 {\scriptsize (2.71)} \\*
 & Gen. Gap & 33.30 {\scriptsize (14.28)} & 22.61 {\scriptsize (3.99)} & 4.47 {\scriptsize (0.75)} & -0.44 {\scriptsize (2.52)} & -1.04 {\scriptsize (2.65)} \\*
 & NC1 & 1.61 {\scriptsize (1.05)} & 0.59 {\scriptsize (0.15)} & 0.39 {\scriptsize (0.03)} & 0.41 {\scriptsize (0.05)} & 0.44 {\scriptsize (0.05)} \\
\midrule
\multirow{3}{1.75cm}{\raggedright Scattering} & Test acc. & 27.78 {\scriptsize (6.00)} & 32.22 {\scriptsize (3.79)} & 75.30 {\scriptsize (0.45)} & 78.94 {\scriptsize (0.58)} & 80.69 {\scriptsize (0.10)} \\*
 & Gen. Gap & 32.46 {\scriptsize (41.34)} & 53.74 {\scriptsize (3.52)} & 14.20 {\scriptsize (1.48)} & 10.39 {\scriptsize (0.78)} & 8.71 {\scriptsize (0.85)} \\*
 & NC1 & 0.93 {\scriptsize (0.10)} & 0.70 {\scriptsize (0.05)} & 0.44 {\scriptsize (0.01)} & 0.43 {\scriptsize (0.02)} & 0.43 {\scriptsize (0.02)} \\
\midrule
\multirow{3}{1.75cm}{\raggedright Fine-tuned Scat. (L2=0)} & Test acc. & \textbf{64.77} {\scriptsize (0.87)} & \textbf{70.20} {\scriptsize (2.07)} & \textbf{83.22} {\scriptsize (0.30)} & \textbf{86.67} {\scriptsize (0.36)} & \textbf{88.25} {\scriptsize (0.39)} \\*
 & Gen. Gap & 35.23 {\scriptsize (0.87)} & 28.24 {\scriptsize (1.05)} & 16.78 {\scriptsize (0.30)} & 13.24 {\scriptsize (0.52)} & 11.52 {\scriptsize (0.50)} \\*
 & NC1 & 0.87 {\scriptsize (0.08)} & 0.70 {\scriptsize (0.07)} & 0.57 {\scriptsize (0.09)} & 0.52 {\scriptsize (0.04)} & 0.47 {\scriptsize (0.04)} \\
\midrule
\multirow{3}{1.75cm}{\raggedright Fine-tuned Scat. (L2=0.1)} & Test acc. & 64.60 {\scriptsize (1.18)} & 70.06 {\scriptsize (0.88)} & 82.12 {\scriptsize (0.69)} & 85.10 {\scriptsize (0.55)} & 86.81 {\scriptsize (0.41)} \\*
 & Gen. Gap & 35.40 {\scriptsize (1.18)} & 29.55 {\scriptsize (0.76)} & 17.69 {\scriptsize (0.92)} & 14.10 {\scriptsize (0.19)} & 11.87 {\scriptsize (1.43)} \\*
 & NC1 & 0.86 {\scriptsize (0.03)} & 0.71 {\scriptsize (0.06)} & 0.52 {\scriptsize (0.04)} & 0.46 {\scriptsize (0.02)} & 0.45 {\scriptsize (0.03)} \\
\midrule
\multirow{3}{1.75cm}{\raggedright Fine-tuned Scat. (L2=1)} & Test acc. & 62.83 {\scriptsize (2.18)} & 68.60 {\scriptsize (0.97)} & 79.73 {\scriptsize (1.03)} & 83.67 {\scriptsize (1.11)} & 85.10 {\scriptsize (0.58)} \\*
 & Gen. Gap & 36.37 {\scriptsize (2.52)} & 29.25 {\scriptsize (0.91)} & 18.92 {\scriptsize (0.82)} & 15.67 {\scriptsize (0.71)} & 14.19 {\scriptsize (0.85)} \\*
 & NC1 & 1.00 {\scriptsize (0.11)} & 0.73 {\scriptsize (0.09)} & 0.52 {\scriptsize (0.04)} & 0.46 {\scriptsize (0.03)} & 0.46 {\scriptsize (0.02)} \\
\midrule
\multirow{3}{1.75cm}{\raggedright ScatNet B.A.} & Test acc. & 43.91 {\scriptsize (1.76)} & 49.90 {\scriptsize (3.66)} & 66.57 {\scriptsize (2.32)} & 71.96 {\scriptsize (0.90)} & 76.20 {\scriptsize (0.07)} \\*
 & Gen. Gap & 56.09 {\scriptsize (1.76)} & 50.10 {\scriptsize (3.66)} & 33.43 {\scriptsize (2.32)} & 28.04 {\scriptsize (0.90)} & 23.80 {\scriptsize (0.07)} \\*
 & NC1 & 0.43 {\scriptsize (0.01)} & 0.56 {\scriptsize (0.09)} & 0.40 {\scriptsize (0.01)} & 0.39 {\scriptsize (0.02)} & 0.39 {\scriptsize (0.02)} \\
\end{longtable}
}%

{%
\captionsetup{list=no}%
\setlength{\LTcapwidth}{\textwidth}%
\begin{longtable}{llcccc}
\caption{Test accuracy, generalization gap and normalized within-class variance in classifier subspace (NC1) across architectures for EuroSat (Part 2 of 2), by training-set fraction. Each cell is the mean over seeds, with the standard deviation in parentheses, set smaller. Best test accuracy per column in bold.}
\label{tab:bestof_eurosat_b}\\
\toprule
Model & Metric & 20\% & 30\% & 50\% & 100\% \\
\midrule
\endfirsthead
\caption*{\textit{Table \thetable{} (part 2 of 2, continued).}}\\
\toprule
Model & Metric & 20\% & 30\% & 50\% & 100\% \\
\midrule
\endhead
\midrule
\multicolumn{6}{r}{\footnotesize\textit{Continued on next page}}\\
\endfoot
\bottomrule
\endlastfoot
\multirow{3}{1.75cm}{\raggedright Random Init. (L2=0)} & Test acc. & 84.44 {\scriptsize (1.18)} & 86.73 {\scriptsize (0.27)} & 88.94 {\scriptsize (0.74)} & 91.98 {\scriptsize (0.53)} \\*
 & Gen. Gap & 15.56 {\scriptsize (1.18)} & 13.26 {\scriptsize (0.28)} & 11.03 {\scriptsize (0.70)} & 7.47 {\scriptsize (0.23)} \\*
 & NC1 & 0.47 {\scriptsize (0.04)} & 0.44 {\scriptsize (0.01)} & 0.40 {\scriptsize (0.01)} & 0.33 {\scriptsize (0.03)} \\
\midrule
\multirow{3}{1.75cm}{\raggedright Random Init. (L2=0.1)} & Test acc. & 85.94 {\scriptsize (0.81)} & 88.00 {\scriptsize (0.83)} & 88.13 {\scriptsize (3.15)} & 85.94 {\scriptsize (4.40)} \\*
 & Gen. Gap & 8.01 {\scriptsize (1.52)} & 5.04 {\scriptsize (0.24)} & 2.19 {\scriptsize (3.14)} & 0.73 {\scriptsize (2.15)} \\*
 & NC1 & 0.41 {\scriptsize (0.02)} & 0.39 {\scriptsize (0.01)} & 0.39 {\scriptsize (0.03)} & 0.37 {\scriptsize (0.02)} \\
\midrule
\multirow{3}{1.75cm}{\raggedright Random Init. (L2=1)} & Test acc. & 60.68 {\scriptsize (1.33)} & 65.31 {\scriptsize (8.32)} & 64.96 {\scriptsize (11.21)} & 66.02 {\scriptsize (11.60)} \\*
 & Gen. Gap & -0.90 {\scriptsize (1.05)} & -0.44 {\scriptsize (1.78)} & -0.40 {\scriptsize (0.31)} & -1.32 {\scriptsize (0.89)} \\*
 & NC1 & 0.39 {\scriptsize (0.01)} & 0.39 {\scriptsize (0.03)} & 0.42 {\scriptsize (0.07)} & 0.37 {\scriptsize (0.06)} \\
\midrule
\multirow{3}{1.75cm}{\raggedright Scattering} & Test acc. & 81.59 {\scriptsize (0.45)} & 83.47 {\scriptsize (0.43)} & 85.24 {\scriptsize (0.80)} & 87.10 {\scriptsize (0.65)} \\*
 & Gen. Gap & 8.09 {\scriptsize (0.31)} & 6.89 {\scriptsize (0.25)} & 5.48 {\scriptsize (1.45)} & 3.51 {\scriptsize (0.80)} \\*
 & NC1 & 0.42 {\scriptsize (0.02)} & 0.42 {\scriptsize (0.01)} & 0.41 {\scriptsize (0.02)} & 0.42 {\scriptsize (0.01)} \\
\midrule
\multirow{3}{1.75cm}{\raggedright Fine-tuned Scat. (L2=0)} & Test acc. & \textbf{89.77} {\scriptsize (0.45)} & \textbf{90.80} {\scriptsize (0.20)} & \textbf{92.86} {\scriptsize (0.25)} & \textbf{94.45} {\scriptsize (0.46)} \\*
 & Gen. Gap & 9.94 {\scriptsize (0.33)} & 9.20 {\scriptsize (0.20)} & 6.95 {\scriptsize (0.05)} & 5.43 {\scriptsize (0.28)} \\*
 & NC1 & 0.47 {\scriptsize (0.03)} & 0.47 {\scriptsize (0.02)} & 0.42 {\scriptsize (0.01)} & 0.38 {\scriptsize (0.01)} \\
\midrule
\multirow{3}{1.75cm}{\raggedright Fine-tuned Scat. (L2=0.1)} & Test acc. & 88.01 {\scriptsize (0.57)} & 89.48 {\scriptsize (0.35)} & 91.70 {\scriptsize (0.25)} & 93.22 {\scriptsize (0.38)} \\*
 & Gen. Gap & 11.48 {\scriptsize (0.44)} & 9.80 {\scriptsize (1.48)} & 7.28 {\scriptsize (1.65)} & 5.30 {\scriptsize (1.54)} \\*
 & NC1 & 0.46 {\scriptsize (0.03)} & 0.45 {\scriptsize (0.02)} & 0.44 {\scriptsize (0.01)} & 0.44 {\scriptsize (0.04)} \\
\midrule
\multirow{3}{1.75cm}{\raggedright Fine-tuned Scat. (L2=1)} & Test acc. & 85.77 {\scriptsize (0.97)} & 87.89 {\scriptsize (0.49)} & 89.52 {\scriptsize (0.57)} & 91.98 {\scriptsize (0.31)} \\*
 & Gen. Gap & 10.48 {\scriptsize (2.84)} & 10.68 {\scriptsize (2.00)} & 8.63 {\scriptsize (1.75)} & 6.70 {\scriptsize (0.60)} \\*
 & NC1 & 0.45 {\scriptsize (0.03)} & 0.45 {\scriptsize (0.05)} & 0.45 {\scriptsize (0.02)} & 0.48 {\scriptsize (0.01)} \\
\midrule
\multirow{3}{1.75cm}{\raggedright ScatNet B.A.} & Test acc. & 80.37 {\scriptsize (0.73)} & 83.49 {\scriptsize (0.96)} & 87.24 {\scriptsize (0.45)} & 91.95 {\scriptsize (0.26)} \\*
 & Gen. Gap & 19.63 {\scriptsize (0.73)} & 16.51 {\scriptsize (0.96)} & 12.76 {\scriptsize (0.45)} & 8.05 {\scriptsize (0.26)} \\*
 & NC1 & 0.36 {\scriptsize (0.01)} & 0.35 {\scriptsize (0.00)} & 0.31 {\scriptsize (0.01)} & 0.27 {\scriptsize (0.01)} \\
\end{longtable}
}%

{%
\setlength{\LTcapwidth}{\textwidth}%
\begin{longtable}{llccccc}
\caption{Test accuracy, generalization gap and normalized within-class variance in classifier subspace (NC1) across architectures for MinC-2500 (Part 1 of 2), by training-set fraction. Each cell is the mean over seeds, with the standard deviation in parentheses, set smaller. Best test accuracy per column in bold.}
\label{tab:bestof_minc2500_a}\\
\toprule
Model & Metric & 0.2\% & 0.5\% & 1\% & 5\% & 10\% \\
\midrule
\endfirsthead
\caption*{\textit{Table \thetable{} (part 1 of 2, continued).}}\\
\toprule
Model & Metric & 0.2\% & 0.5\% & 1\% & 5\% & 10\% \\
\midrule
\endhead
\midrule
\multicolumn{7}{r}{\footnotesize\textit{Continued on next page}}\\
\endfoot
\bottomrule
\endlastfoot
\multirow{3}{1.75cm}{\raggedright Random Init. (L2=0)} & Test acc. & 11.93 {\scriptsize (1.34)} & 15.46 {\scriptsize (0.72)} & 19.50 {\scriptsize (0.56)} & 28.40 {\scriptsize (0.30)} & 33.88 {\scriptsize (0.60)} \\*
 & Gen. Gap & 82.64 {\scriptsize (3.56)} & 77.15 {\scriptsize (4.67)} & 62.34 {\scriptsize (8.63)} & 38.96 {\scriptsize (9.96)} & 31.43 {\scriptsize (7.32)} \\*
 & NC1 & 6.28 {\scriptsize (1.96)} & 6.38 {\scriptsize (0.60)} & 5.94 {\scriptsize (0.68)} & 3.59 {\scriptsize (0.42)} & 2.90 {\scriptsize (0.18)} \\
\midrule
\multirow{3}{1.75cm}{\raggedright Random Init. (L2=0.1)} & Test acc. & 10.87 {\scriptsize (1.32)} & 14.28 {\scriptsize (1.32)} & 19.23 {\scriptsize (1.70)} & 35.64 {\scriptsize (0.85)} & 41.57 {\scriptsize (0.37)} \\*
 & Gen. Gap & 84.34 {\scriptsize (2.20)} & 75.61 {\scriptsize (6.87)} & 48.63 {\scriptsize (2.49)} & 31.65 {\scriptsize (9.09)} & 17.50 {\scriptsize (3.41)} \\*
 & NC1 & 8.38 {\scriptsize (3.33)} & 7.62 {\scriptsize (1.46)} & 6.50 {\scriptsize (0.58)} & 2.77 {\scriptsize (0.02)} & 2.29 {\scriptsize (0.04)} \\
\midrule
\multirow{3}{1.75cm}{\raggedright Random Init. (L2=1)} & Test acc. & 8.96 {\scriptsize (3.52)} & 10.33 {\scriptsize (0.89)} & 8.95 {\scriptsize (0.48)} & 36.34 {\scriptsize (1.00)} & 35.79 {\scriptsize (8.83)} \\*
 & Gen. Gap & 61.04 {\scriptsize (36.98)} & 65.11 {\scriptsize (3.97)} & 10.57 {\scriptsize (5.45)} & 9.54 {\scriptsize (2.22)} & 3.29 {\scriptsize (0.87)} \\*
 & NC1 & 9.48 {\scriptsize (1.79)} & 7.20 {\scriptsize (1.64)} & 4.09 {\scriptsize (1.03)} & 2.63 {\scriptsize (0.08)} & 2.51 {\scriptsize (0.08)} \\
\midrule
\multirow{3}{1.75cm}{\raggedright Scattering} & Test acc. & 8.95 {\scriptsize (0.95)} & 18.24 {\scriptsize (0.83)} & 20.26 {\scriptsize (0.24)} & 29.71 {\scriptsize (0.37)} & 33.48 {\scriptsize (0.47)} \\*
 & Gen. Gap & 61.92 {\scriptsize (1.75)} & 47.62 {\scriptsize (2.89)} & 24.59 {\scriptsize (3.27)} & 12.75 {\scriptsize (0.72)} & 8.67 {\scriptsize (1.42)} \\*
 & NC1 & 5.40 {\scriptsize (0.39)} & 6.18 {\scriptsize (0.48)} & 5.78 {\scriptsize (0.36)} & 4.62 {\scriptsize (0.11)} & 4.36 {\scriptsize (0.03)} \\
\midrule
\multirow{3}{1.75cm}{\raggedright Fine-tuned Scat. (L2=0)} & Test acc. & 15.61 {\scriptsize (1.10)} & \textbf{20.22} {\scriptsize (0.93)} & 24.05 {\scriptsize (2.52)} & 36.68 {\scriptsize (1.00)} & 40.73 {\scriptsize (0.42)} \\*
 & Gen. Gap & 75.05 {\scriptsize (10.50)} & 70.33 {\scriptsize (6.12)} & 52.17 {\scriptsize (13.84)} & 51.41 {\scriptsize (5.63)} & 40.33 {\scriptsize (7.65)} \\*
 & NC1 & 6.50 {\scriptsize (0.92)} & 6.30 {\scriptsize (0.52)} & 5.43 {\scriptsize (0.50)} & 3.77 {\scriptsize (0.23)} & 3.36 {\scriptsize (0.09)} \\
\midrule
\multirow{3}{1.75cm}{\raggedright Fine-tuned Scat. (L2=0.1)} & Test acc. & \textbf{15.84} {\scriptsize (1.05)} & 19.99 {\scriptsize (1.04)} & 24.38 {\scriptsize (1.49)} & 35.29 {\scriptsize (1.44)} & 39.98 {\scriptsize (1.57)} \\*
 & Gen. Gap & 74.38 {\scriptsize (10.42)} & 67.40 {\scriptsize (5.01)} & 57.04 {\scriptsize (4.07)} & 36.57 {\scriptsize (9.33)} & 26.97 {\scriptsize (10.16)} \\*
 & NC1 & 6.21 {\scriptsize (0.65)} & 6.07 {\scriptsize (0.56)} & 5.27 {\scriptsize (0.43)} & 3.76 {\scriptsize (0.20)} & 3.45 {\scriptsize (0.15)} \\
\midrule
\multirow{3}{1.75cm}{\raggedright Fine-tuned Scat. (L2=1)} & Test acc. & 15.48 {\scriptsize (0.99)} & 20.10 {\scriptsize (0.78)} & 22.44 {\scriptsize (2.87)} & 34.78 {\scriptsize (0.55)} & 39.00 {\scriptsize (1.47)} \\*
 & Gen. Gap & 70.82 {\scriptsize (7.09)} & 65.00 {\scriptsize (3.28)} & 36.13 {\scriptsize (19.33)} & 40.99 {\scriptsize (0.82)} & 24.32 {\scriptsize (8.36)} \\*
 & NC1 & 6.23 {\scriptsize (0.61)} & 6.10 {\scriptsize (0.51)} & 5.04 {\scriptsize (0.34)} & 4.05 {\scriptsize (0.11)} & 3.66 {\scriptsize (0.11)} \\
\midrule
\multirow{3}{1.75cm}{\raggedright ScatNet B.A.} & Test acc. & 15.00 {\scriptsize (0.92)} & 17.91 {\scriptsize (1.40)} & \textbf{25.56} {\scriptsize (1.16)} & \textbf{44.23} {\scriptsize (0.64)} & \textbf{51.41} {\scriptsize (0.35)} \\*
 & Gen. Gap & 47.83 {\scriptsize (6.26)} & 27.52 {\scriptsize (7.81)} & 25.76 {\scriptsize (1.12)} & 19.74 {\scriptsize (1.72)} & 16.47 {\scriptsize (0.14)} \\*
 & NC1 & 3.95 {\scriptsize (0.56)} & 3.79 {\scriptsize (0.45)} & 3.02 {\scriptsize (0.28)} & 2.12 {\scriptsize (0.05)} & 1.86 {\scriptsize (0.03)} \\
\end{longtable}
}%

{%
\captionsetup{list=no}%
\setlength{\LTcapwidth}{\textwidth}%
\begin{longtable}{llcccc}
\caption{Test accuracy, generalization gap and normalized within-class variance in classifier subspace (NC1) across architectures for MinC-2500 (Part 2 of 2), by training-set fraction. Each cell is the mean over seeds, with the standard deviation in parentheses, set smaller. Best test accuracy per column in bold.}
\label{tab:bestof_minc2500_b}\\
\toprule
Model & Metric & 20\% & 30\% & 50\% & 100\% \\
\midrule
\endfirsthead
\caption*{\textit{Table \thetable{} (part 2 of 2, continued).}}\\
\toprule
Model & Metric & 20\% & 30\% & 50\% & 100\% \\
\midrule
\endhead
\midrule
\multicolumn{6}{r}{\footnotesize\textit{Continued on next page}}\\
\endfoot
\bottomrule
\endlastfoot
\multirow{3}{1.75cm}{\raggedright Random Init. (L2=0)} & Test acc. & 40.81 {\scriptsize (0.38)} & 44.35 {\scriptsize (0.73)} & 49.86 {\scriptsize (0.99)} & 55.36 {\scriptsize (1.21)} \\*
 & Gen. Gap & 28.29 {\scriptsize (6.62)} & 23.53 {\scriptsize (4.60)} & 13.94 {\scriptsize (7.40)} & 8.13 {\scriptsize (2.77)} \\*
 & NC1 & 2.70 {\scriptsize (0.06)} & 2.56 {\scriptsize (0.11)} & 2.34 {\scriptsize (0.01)} & 2.09 {\scriptsize (0.08)} \\
\midrule
\multirow{3}{1.75cm}{\raggedright Random Init. (L2=0.1)} & Test acc. & 40.65 {\scriptsize (12.06)} & 49.21 {\scriptsize (0.86)} & 50.37 {\scriptsize (0.74)} & 50.69 {\scriptsize (0.20)} \\*
 & Gen. Gap & 1.69 {\scriptsize (2.01)} & -0.03 {\scriptsize (0.94)} & -2.37 {\scriptsize (0.72)} & -4.05 {\scriptsize (0.33)} \\*
 & NC1 & 2.47 {\scriptsize (0.93)} & 1.98 {\scriptsize (0.11)} & 2.19 {\scriptsize (0.08)} & 2.54 {\scriptsize (0.13)} \\
\midrule
\multirow{3}{1.75cm}{\raggedright Random Init. (L2=1)} & Test acc. & 36.22 {\scriptsize (10.30)} & 37.67 {\scriptsize (8.21)} & 32.75 {\scriptsize (8.47)} & 26.43 {\scriptsize (0.63)} \\*
 & Gen. Gap & 0.27 {\scriptsize (2.94)} & -0.58 {\scriptsize (3.62)} & 2.08 {\scriptsize (4.10)} & 6.14 {\scriptsize (1.16)} \\*
 & NC1 & 2.79 {\scriptsize (0.38)} & 2.73 {\scriptsize (0.41)} & 2.92 {\scriptsize (0.33)} & 2.96 {\scriptsize (0.30)} \\
\midrule
\multirow{3}{1.75cm}{\raggedright Scattering} & Test acc. & 35.97 {\scriptsize (0.19)} & 37.53 {\scriptsize (0.38)} & 39.36 {\scriptsize (0.25)} & 40.57 {\scriptsize (0.43)} \\*
 & Gen. Gap & 4.38 {\scriptsize (0.16)} & 2.93 {\scriptsize (0.51)} & 1.06 {\scriptsize (0.76)} & 0.37 {\scriptsize (0.11)} \\*
 & NC1 & 4.23 {\scriptsize (0.10)} & 4.14 {\scriptsize (0.03)} & 4.07 {\scriptsize (0.11)} & 3.98 {\scriptsize (0.06)} \\
\midrule
\multirow{3}{1.75cm}{\raggedright Fine-tuned Scat. (L2=0)} & Test acc. & 45.19 {\scriptsize (0.77)} & 47.91 {\scriptsize (0.64)} & 51.63 {\scriptsize (0.70)} & 57.19 {\scriptsize (0.45)} \\*
 & Gen. Gap & 29.33 {\scriptsize (6.93)} & 33.10 {\scriptsize (5.39)} & 30.60 {\scriptsize (1.89)} & 20.20 {\scriptsize (1.07)} \\*
 & NC1 & 3.08 {\scriptsize (0.03)} & 2.92 {\scriptsize (0.05)} & 2.66 {\scriptsize (0.05)} & 2.32 {\scriptsize (0.01)} \\
\midrule
\multirow{3}{1.75cm}{\raggedright Fine-tuned Scat. (L2=0.1)} & Test acc. & 45.36 {\scriptsize (0.60)} & 47.30 {\scriptsize (0.75)} & 50.09 {\scriptsize (0.60)} & 53.28 {\scriptsize (0.33)} \\*
 & Gen. Gap & 26.54 {\scriptsize (1.14)} & 21.14 {\scriptsize (1.64)} & 12.72 {\scriptsize (0.53)} & 5.21 {\scriptsize (0.40)} \\*
 & NC1 & 2.97 {\scriptsize (0.09)} & 2.76 {\scriptsize (0.08)} & 2.53 {\scriptsize (0.02)} & 2.31 {\scriptsize (0.05)} \\
\midrule
\multirow{3}{1.75cm}{\raggedright Fine-tuned Scat. (L2=1)} & Test acc. & 43.57 {\scriptsize (0.53)} & 45.77 {\scriptsize (0.83)} & 47.39 {\scriptsize (0.96)} & 50.03 {\scriptsize (0.48)} \\*
 & Gen. Gap & 17.66 {\scriptsize (0.64)} & 13.30 {\scriptsize (0.64)} & 8.21 {\scriptsize (1.36)} & 3.02 {\scriptsize (0.92)} \\*
 & NC1 & 3.22 {\scriptsize (0.08)} & 3.02 {\scriptsize (0.13)} & 2.83 {\scriptsize (0.03)} & 2.63 {\scriptsize (0.03)} \\
\midrule
\multirow{3}{1.75cm}{\raggedright ScatNet B.A.} & Test acc. & \textbf{58.17} {\scriptsize (0.43)} & \textbf{61.37} {\scriptsize (0.40)} & \textbf{65.25} {\scriptsize (0.53)} & \textbf{69.61} {\scriptsize (0.68)} \\*
 & Gen. Gap & 12.30 {\scriptsize (1.81)} & 9.14 {\scriptsize (1.08)} & 5.66 {\scriptsize (1.04)} & 3.22 {\scriptsize (1.19)} \\*
 & NC1 & 1.64 {\scriptsize (0.06)} & 1.52 {\scriptsize (0.05)} & 1.43 {\scriptsize (0.02)} & 1.35 {\scriptsize (0.04)} \\
\end{longtable}
}%

{%
\setlength{\LTcapwidth}{\textwidth}%
\begin{longtable}{llcccc}
\caption{Test accuracy, generalization gap and normalized within-class variance in classifier subspace (NC1) across architectures for SuroTex (Part 1 of 2), by training-set fraction. Each cell is the mean over seeds, with the standard deviation in parentheses, set smaller. Best test accuracy per column in bold.}
\label{tab:bestof_surotex_a}\\
\toprule
Model & Metric & 2\% & 5\% & 10\% & 15\% \\
\midrule
\endfirsthead
\caption*{\textit{Table \thetable{} (part 1 of 2, continued).}}\\
\toprule
Model & Metric & 2\% & 5\% & 10\% & 15\% \\
\midrule
\endhead
\midrule
\multicolumn{6}{r}{\footnotesize\textit{Continued on next page}}\\
\endfoot
\bottomrule
\endlastfoot
\multirow{3}{1.75cm}{\raggedright Random Init. (L2=0)} & Test acc. & 49.00 {\scriptsize (0.95)} & 77.93 {\scriptsize (4.59)} & 87.17 {\scriptsize (2.37)} & 92.90 {\scriptsize (1.57)} \\*
 & Gen. Gap & 51.00 {\scriptsize (0.95)} & 20.29 {\scriptsize (4.34)} & 9.72 {\scriptsize (2.51)} & 5.32 {\scriptsize (1.79)} \\*
 & NC1 & 0.48 {\scriptsize (0.06)} & 0.33 {\scriptsize (0.02)} & 0.30 {\scriptsize (0.03)} & 0.32 {\scriptsize (0.02)} \\
\midrule
\multirow{3}{1.75cm}{\raggedright Random Init. (L2=0.1)} & Test acc. & 39.80 {\scriptsize (21.30)} & 77.00 {\scriptsize (2.46)} & 87.43 {\scriptsize (1.85)} & 92.07 {\scriptsize (2.33)} \\*
 & Gen. Gap & 49.53 {\scriptsize (3.31)} & 20.78 {\scriptsize (1.23)} & 10.12 {\scriptsize (2.83)} & 6.23 {\scriptsize (1.86)} \\*
 & NC1 & 0.43 {\scriptsize (0.13)} & 0.33 {\scriptsize (0.03)} & 0.32 {\scriptsize (0.01)} & 0.34 {\scriptsize (0.02)} \\
\midrule
\multirow{3}{1.75cm}{\raggedright Random Init. (L2=1)} & Test acc. & 23.93 {\scriptsize (21.40)} & 75.37 {\scriptsize (2.95)} & 84.57 {\scriptsize (1.20)} & 70.27 {\scriptsize (16.19)} \\*
 & Gen. Gap & 54.07 {\scriptsize (4.02)} & 16.63 {\scriptsize (8.09)} & 10.88 {\scriptsize (1.55)} & 14.99 {\scriptsize (6.81)} \\*
 & NC1 & 0.30 {\scriptsize (0.19)} & 0.30 {\scriptsize (0.02)} & 0.31 {\scriptsize (0.04)} & 0.31 {\scriptsize (0.02)} \\
\midrule
\multirow{3}{1.75cm}{\raggedright Scattering} & Test acc. & 34.73 {\scriptsize (1.31)} & 55.07 {\scriptsize (1.79)} & 74.60 {\scriptsize (2.98)} & 83.17 {\scriptsize (0.42)} \\*
 & Gen. Gap & 37.93 {\scriptsize (6.61)} & 5.82 {\scriptsize (2.70)} & -2.38 {\scriptsize (3.84)} & -3.02 {\scriptsize (1.34)} \\*
 & NC1 & 0.12 {\scriptsize (0.01)} & 0.12 {\scriptsize (0.00)} & 0.12 {\scriptsize (0.01)} & 0.12 {\scriptsize (0.01)} \\
\midrule
\multirow{3}{1.75cm}{\raggedright Fine-tuned Scat. (L2=0)} & Test acc. & \textbf{60.13} {\scriptsize (3.67)} & 82.43 {\scriptsize (3.65)} & 91.83 {\scriptsize (1.16)} & 95.07 {\scriptsize (1.07)} \\*
 & Gen. Gap & 37.20 {\scriptsize (3.73)} & 12.01 {\scriptsize (2.70)} & 4.50 {\scriptsize (1.73)} & 2.19 {\scriptsize (1.38)} \\*
 & NC1 & 0.15 {\scriptsize (0.01)} & 0.16 {\scriptsize (0.01)} & 0.17 {\scriptsize (0.01)} & 0.16 {\scriptsize (0.01)} \\
\midrule
\multirow{3}{1.75cm}{\raggedright Fine-tuned Scat. (L2=0.1)} & Test acc. & 59.60 {\scriptsize (3.63)} & 81.40 {\scriptsize (4.00)} & 91.00 {\scriptsize (1.31)} & 94.03 {\scriptsize (1.01)} \\*
 & Gen. Gap & 38.40 {\scriptsize (5.29)} & 10.82 {\scriptsize (3.79)} & 3.78 {\scriptsize (2.35)} & 2.48 {\scriptsize (1.18)} \\*
 & NC1 & 0.15 {\scriptsize (0.01)} & 0.14 {\scriptsize (0.00)} & 0.13 {\scriptsize (0.01)} & 0.13 {\scriptsize (0.01)} \\
\midrule
\multirow{3}{1.75cm}{\raggedright Fine-tuned Scat. (L2=1)} & Test acc. & 58.90 {\scriptsize (3.89)} & 80.60 {\scriptsize (4.65)} & 89.43 {\scriptsize (1.57)} & 93.47 {\scriptsize (0.87)} \\*
 & Gen. Gap & 39.77 {\scriptsize (6.17)} & 11.40 {\scriptsize (3.57)} & 3.46 {\scriptsize (1.87)} & 1.27 {\scriptsize (0.66)} \\*
 & NC1 & 0.13 {\scriptsize (0.01)} & 0.13 {\scriptsize (0.00)} & 0.12 {\scriptsize (0.01)} & 0.12 {\scriptsize (0.00)} \\
\midrule
\multirow{3}{1.75cm}{\raggedright ScatNet B.A.} & Test acc. & 55.80 {\scriptsize (3.50)} & \textbf{82.73} {\scriptsize (2.75)} & \textbf{93.70} {\scriptsize (1.25)} & \textbf{95.20} {\scriptsize (1.18)} \\*
 & Gen. Gap & 41.53 {\scriptsize (2.45)} & 13.04 {\scriptsize (2.59)} & 4.19 {\scriptsize (0.74)} & 2.65 {\scriptsize (1.39)} \\*
 & NC1 & 0.19 {\scriptsize (0.01)} & 0.16 {\scriptsize (0.01)} & 0.15 {\scriptsize (0.01)} & 0.15 {\scriptsize (0.01)} \\
\end{longtable}
}%

{%
\captionsetup{list=no}%
\setlength{\LTcapwidth}{\textwidth}%
\begin{longtable}{llcccc}
\caption{Test accuracy, generalization gap and normalized within-class variance in classifier subspace (NC1) across architectures for SuroTex (Part 2 of 2), by training-set fraction. Each cell is the mean over seeds, with the standard deviation in parentheses, set smaller. Best test accuracy per column in bold.}
\label{tab:bestof_surotex_b}\\
\toprule
Model & Metric & 20\% & 30\% & 50\% & 100\% \\
\midrule
\endfirsthead
\caption*{\textit{Table \thetable{} (part 2 of 2, continued).}}\\
\toprule
Model & Metric & 20\% & 30\% & 50\% & 100\% \\
\midrule
\endhead
\midrule
\multicolumn{6}{r}{\footnotesize\textit{Continued on next page}}\\
\endfoot
\bottomrule
\endlastfoot
\multirow{3}{1.75cm}{\raggedright Random Init. (L2=0)} & Test acc. & 94.20 {\scriptsize (0.62)} & 96.00 {\scriptsize (1.73)} & 97.23 {\scriptsize (0.96)} & 98.77 {\scriptsize (0.83)} \\*
 & Gen. Gap & 3.74 {\scriptsize (0.92)} & 1.65 {\scriptsize (0.20)} & 0.68 {\scriptsize (1.71)} & -0.54 {\scriptsize (0.31)} \\*
 & NC1 & 0.30 {\scriptsize (0.02)} & 0.28 {\scriptsize (0.03)} & 0.28 {\scriptsize (0.02)} & 0.26 {\scriptsize (0.02)} \\
\midrule
\multirow{3}{1.75cm}{\raggedright Random Init. (L2=0.1)} & Test acc. & 94.37 {\scriptsize (1.60)} & 95.57 {\scriptsize (1.64)} & 97.20 {\scriptsize (1.28)} & 98.10 {\scriptsize (0.70)} \\*
 & Gen. Gap & 3.30 {\scriptsize (2.05)} & 0.54 {\scriptsize (1.33)} & 0.49 {\scriptsize (1.73)} & 0.12 {\scriptsize (0.72)} \\*
 & NC1 & 0.32 {\scriptsize (0.02)} & 0.28 {\scriptsize (0.03)} & 0.31 {\scriptsize (0.04)} & 0.27 {\scriptsize (0.01)} \\
\midrule
\multirow{3}{1.75cm}{\raggedright Random Init. (L2=1)} & Test acc. & 92.17 {\scriptsize (0.93)} & 93.70 {\scriptsize (1.49)} & 95.40 {\scriptsize (1.41)} & 88.23 {\scriptsize (7.26)} \\*
 & Gen. Gap & 3.39 {\scriptsize (1.05)} & 1.67 {\scriptsize (0.96)} & 0.54 {\scriptsize (1.57)} & 1.44 {\scriptsize (2.74)} \\*
 & NC1 & 0.30 {\scriptsize (0.02)} & 0.28 {\scriptsize (0.02)} & 0.27 {\scriptsize (0.02)} & 0.24 {\scriptsize (0.04)} \\
\midrule
\multirow{3}{1.75cm}{\raggedright Scattering} & Test acc. & 86.37 {\scriptsize (1.50)} & 89.67 {\scriptsize (1.10)} & 91.97 {\scriptsize (1.02)} & 94.60 {\scriptsize (0.17)} \\*
 & Gen. Gap & -4.53 {\scriptsize (2.35)} & -5.00 {\scriptsize (2.00)} & -6.40 {\scriptsize (1.30)} & -6.63 {\scriptsize (0.29)} \\*
 & NC1 & 0.12 {\scriptsize (0.01)} & 0.12 {\scriptsize (0.01)} & 0.12 {\scriptsize (0.00)} & 0.12 {\scriptsize (0.00)} \\
\midrule
\multirow{3}{1.75cm}{\raggedright Fine-tuned Scat. (L2=0)} & Test acc. & 95.40 {\scriptsize (0.78)} & 96.60 {\scriptsize (1.06)} & 97.50 {\scriptsize (0.95)} & 98.87 {\scriptsize (0.46)} \\*
 & Gen. Gap & 1.60 {\scriptsize (0.10)} & 0.84 {\scriptsize (1.12)} & 0.04 {\scriptsize (1.67)} & -0.03 {\scriptsize (0.78)} \\*
 & NC1 & 0.16 {\scriptsize (0.01)} & 0.17 {\scriptsize (0.01)} & 0.17 {\scriptsize (0.01)} & 0.15 {\scriptsize (0.01)} \\
\midrule
\multirow{3}{1.75cm}{\raggedright Fine-tuned Scat. (L2=0.1)} & Test acc. & 95.07 {\scriptsize (0.23)} & 97.40 {\scriptsize (1.55)} & 97.63 {\scriptsize (1.48)} & 98.00 {\scriptsize (0.62)} \\*
 & Gen. Gap & 1.66 {\scriptsize (1.03)} & 0.11 {\scriptsize (1.27)} & -0.63 {\scriptsize (0.51)} & -1.84 {\scriptsize (0.83)} \\*
 & NC1 & 0.13 {\scriptsize (0.00)} & 0.14 {\scriptsize (0.01)} & 0.13 {\scriptsize (0.00)} & 0.13 {\scriptsize (0.01)} \\
\midrule
\multirow{3}{1.75cm}{\raggedright Fine-tuned Scat. (L2=1)} & Test acc. & 94.70 {\scriptsize (1.00)} & 96.37 {\scriptsize (0.61)} & 96.33 {\scriptsize (0.49)} & 97.63 {\scriptsize (0.15)} \\*
 & Gen. Gap & -0.64 {\scriptsize (2.15)} & -1.63 {\scriptsize (0.80)} & -1.56 {\scriptsize (2.18)} & -2.27 {\scriptsize (1.22)} \\*
 & NC1 & 0.12 {\scriptsize (0.00)} & 0.13 {\scriptsize (0.01)} & 0.13 {\scriptsize (0.00)} & 0.13 {\scriptsize (0.00)} \\
\midrule
\multirow{3}{1.75cm}{\raggedright ScatNet B.A.} & Test acc. & \textbf{96.73} {\scriptsize (0.35)} & \textbf{97.83} {\scriptsize (0.15)} & \textbf{98.50} {\scriptsize (0.44)} & \textbf{99.47} {\scriptsize (0.31)} \\*
 & Gen. Gap & 0.71 {\scriptsize (0.52)} & 0.31 {\scriptsize (0.36)} & 0.10 {\scriptsize (0.50)} & -0.25 {\scriptsize (0.32)} \\*
 & NC1 & 0.15 {\scriptsize (0.01)} & 0.15 {\scriptsize (0.01)} & 0.15 {\scriptsize (0.00)} & 0.14 {\scriptsize (0.01)} \\
\end{longtable}
}%

\end{document}